\documentclass{article}

\def\inclapp{0}
\def\viewchanges{1}
\def\viewauthors{1}
\def\viewkeywords{0}
\def\usehyperlinks{1}
\def\addackn{0}
\def\preprint{1}

\usepackage[dvipsnames]{xcolor}

\if\preprint1
    \usepackage[preprint]{tmlr}
\else
    \if\viewauthors1
        \usepackage[accepted]{tmlr}
    \else
        \usepackage{tmlr}
    \fi
\fi

\newcommand{\myAND}{\\}
\let\myAND\AND
\let\AND\undefined

\usepackage{natbib}

\usepackage[utf8]{inputenc} 
\usepackage[T1]{fontenc}    
\usepackage{booktabs}       
\usepackage{amsfonts}       
\usepackage{nicefrac}       
\usepackage{microtype}      
\usepackage{amsthm}
\usepackage{amsmath}
\usepackage{amssymb}
\usepackage{enumitem}
\usepackage{bbm}
\usepackage{dsfont}
\usepackage{url}            

\usepackage{breakurl}
\if\usehyperlinks1
	\usepackage[colorlinks, citecolor=blue, linkcolor=magenta, breaklinks]{hyperref}       
\fi

\usepackage{graphicx,wrapfig,lipsum}
\usepackage{subcaption}
\usepackage{verbatim}
\usepackage{bm}
\usepackage{adjustbox}
\usepackage{footnote}
\usepackage{algorithm}
\usepackage{algpseudocode}
\usepackage{tikz}
\usepackage{cleveref}
\usepackage{multirow}

\if\inclapp0
	\usepackage{xr}
\fi

\newtheorem{theorem}{Theorem}[section]
\newtheorem{lemma}[theorem]{Lemma}
\newtheorem{definition}[theorem]{Definition}

\newtheorem{rem}[theorem]{Remark}
\newtheorem{cor}[theorem]{Corollary}

\newtheorem{assumption}{Assumption}

\crefname{assumption}{Assumption}{Assumptions}
\Crefname{assumption}{Assumption}{Assumptions}

\let\P\undefined%
\newcommand{\P}{\mathbb{P}}
\newcommand{\F}{\mathcal{F}}

\newcommand{\E}{\mathbb{E}} 
\newcommand{\N}{\mathbb{N}} 
\newcommand{\R}{\mathbb{R}} 
\newcommand{\C}{\mathcal{H}} 
\newcommand{\argmin}{\operatorname{argmin}}

\newcommand{\omb}{(\omega)}

\let\del\undefined
\let\com\undefined

\if\viewchanges1
	
	\newcommand{\del}[1]{{\color{red}{#1}}}
	\newcommand{\com}[1]{{\color{orange}{#1}}}
	
\else
	
	\newcommand{\del}[1]{}
	\newcommand{\com}[1]{}
	
\fi

\allowdisplaybreaks

\title{{Operator Neural Jump ODEs: $L^2$-optimal prediction in function spaces}}

\if\viewauthors1
	\author{%
        \name Florian Krach \email{florian.krach@me.com} \\
        \addr Department of Mathematics\\
        ETH Zurich
        \myAND 
         \name Oliver Löthgren \email{olothgren@hotmail.com} \\
        \addr Department of Mathematics\\
        ETH Zurich
        \myAND 
         \name Josef Teichmann \email{josef.teichmann@math.ethz.ch} \\
        \addr Department of Mathematics\\
        ETH Zurich
	}
\else
	\author{}
\fi

\providecommand{\keywords}[1]{\textbf{{Keywords:}} \textit{#1}}

\begin{document}

\maketitle

\begin{abstract}
In this paper, we study the extension of Neural Jump ODEs \citep[NJ-ODEs;][]{krach2022optimal} to infinite-dimensional function spaces. In particular, the underlying process $X$ now takes values in $L^2(\Xi, \R^{d_X})$ instead of $\R^{d_X}$ and the Operator NJ-ODE approximates the optimal predictor of this process by producing a representative of the conditional expectation.
The NJ-ODE model is a framework for online learning the optimal prediction of continuous-time stochastic processes, given discrete, possibly irregular and incomplete past observations. In a series of works, this model has been extended to deal with generic path-dependent processes, with observation noise and dependent observations, with long-term predictions, and with input-output systems. However, throughout all of these works, the underlying processes were restricted to be finite-dimensional. In particular, function-valued problems, like yield curve or volatility surface predictions, could only be handled through discretization, which inherently leads to a loss of information. In this work, we build on ideas from Neural Operator methods that allow us to extend the NJ-ODE framework to an infinite-dimensional output process.
To prove convergence of the NJ-ODE to the optimal prediction process, we develop a new approximation strategy that also generalizes previous works in the finite-dimensional setting by considerably weakening the assumptions.
\end{abstract}

\if\viewkeywords
\fi

\section{Introduction}\label{sec:Introduction}

An incredibly diverse range of players participate in financial markets. From everyday retail investors to profit-seeking entities like prop-trading firms, asset managers, and hedge funds, to institutional players like banks, they all operate with distinct objectives. Yet, in their pursuit of profits, they all share a unifying quality: reliance on a ubiquitous task in quantitative finance, namely forecasting. 

The forecasting problem entails finding optimal future predictions given currently available information, and there are two categories of solutions. While offline forecasting makes predictions based on time series in their entirety, online forecasting explicitly considers the sequential nature of the data, processing information as it is observed, and updating labels dynamically in time. The focus of this paper will be on the latter. In particular, we consider the general setting of infinite-dimensional, continuous-time stochastic processes with unknown dynamics and estimate their $L^2$-optimal predictor, i.e.~the conditional expectation given past observations. Modeling this auxiliary stochastic process circumvents the need to learn the dynamics of the underlying process itself, which is a more complicated task. 

Several architectures have been implemented to this end, including latent-space models such as RNNs \citep{rumelhart1985learning,jordan1997serial}, state-space models such as Kalman \citep{kalman1960new} or particle filters \citep{djuric2003particle}, as well as their numerous extensions \citep{archer2015black,karl2016deep,krishnan2015deep,krishnan2017structured,naesseth2018variational,revach2022kalmannet}. These methods estimate either the conditional expectation directly or via the conditional probability distribution. However, they are primarily designed for discrete-time prediction and cannot be applied effectively to partially observed or irregularly sampled data. Moreover, methods such as the Kalman filter make strong distributional assumptions that limit the scope of applicability.

To achieve full generality, a model should therefore be nonparametric, and elevated to continuous-time considering the latent dynamics between observations. A natural extension is to learn the differential equations that govern these dynamics with neural networks. For example, \citet{Kidger2020NeuralCD} develop a Neural CDE that extends the Neural ODE \citep{chen2018neural} as the continuous-time analogue of an RNN. While Neural ODEs depend only on data through initial conditions, the Neural CDE introduces data dependence through time by using interpolated observations as a driver. Meanwhile, the ODE-RNN \citep{ODERNN2019} builds directly on the RNN by incorporating a Neural ODE between observations. A special case of this architecture is the GRU-ODE-Bayes \citep{Brouwer2019GRUODEBayesCM}, which replaces the RNN with a gated recurrent unit (GRU) and the neural ODE with a continuous GRU. The former paper establishes a strong approximation theory, but the latter is merely studied empirically. Moreover, both lack a rigorous mathematical training framework and proofs of consistency. The Neural Jump ODE (NJ-ODE) \citep{herrera2021neural} and its extensions address this. It provides a mathematical problem statement and a compatible objective function that allows for proofs of convergence in the $L^2$-norm. Initially introduced as a modification of the ODE-RNN, with a rigorous mathematical problem setting and a new objective function, it has been extended to account for path dependence \citep{krach2022optimal}, noisy data and a dependent observation framework \citep{NJODE3}, long-term predictions \citep{krach2024learning} and input-output systems \citep{heiss2024nonparametric}. The previous models all assumed continuity of the function (and its generalized derivative) that retrieves the conditional expectation\footnote{Such a function exists by the Doob-Dynkin Lemma but the authors assume existence of its generalized derivative}. In this paper, we show that the continuity of these functions is not a necessary assumption. In fact, the weaker assumption of measurability is sufficient and allows for a wider range of applications. This is achieved by relying on RNNs instead of signatures to capture path-dependence in the theoretical results. 

Despite their flexibility in handling irregular observation frameworks, the aforementioned models remain confined to finite-dimensional settings. This limitation excludes a broad class of applications in quantitative finance and beyond, where objects are inherently function-valued, such as the forecasting of yield curves or implied volatility surfaces. Naturally, the additional degrees of freedom make this a significantly more challenging task. Typical methods for function-valued forecasting include Gaussian Processes \citep{Rasmussen2004} and kernel-based methods \citep{pmlr-v2-weinberger07a, pmlr-v9-kadri10a}, which face poor scaling in calculating kernel matrices and impose restrictive assumptions on the underlying process. Alternatively, one can use a dimension reduction method to project the infinite-dimensional problem to a finite-dimensional one. However, this results in information loss and sacrifices the function-valued forecast. In parallel, the field of operator learning has seen significant growth, especially with the advent of Neural Operators \citep{kovachki2023neural}, originally designed for solving PDEs. These generalize feedforward neural networks by replacing affine transformations with integral kernel operators. From discrete data, the kernels capture spatial information and enable the learning of maps between function spaces, but these maps are static and not suitable for online learning. Motivated by this approach, we introduce the Operator Neural Jump ODE (O-NJ-ODE), which incorporates a generalized kernel operator into the NJ-ODE framework and enables the online learning of function-valued forecasts. Moreover, by adjusting the mathematical problem formulation accordingly, we establish theoretical guarantees by proving convergence to the $L^2$-optimal predictor, i.e. the conditional expectation equivalence class, in the pseudo-metric $d_k$ of Definition \ref{def:indistinguishability}. 

\subsection{Outline of the Paper}
We begin by formulating the problem setting and stating the necessary assumptions in \Cref{sec:Problem Setting}, before introducing the O-NJ-ODE model and training framework in \Cref{sec:proposed method}. Once the architecture is fully established, \Cref{sec:Convergence Guarantees} presents the theoretical guarantees along with their proofs, including convergence of the O-NJ-ODE in training with respect to the proposed objective function, and the practically relevant convergence of the Monte Carlo approximation of this function to its true value.

\section{Problem Setting}\label{sec:Problem Setting}
We assume to have a continuous-time stochastic process $X$ taking values in an $L^2$-Bochner space, where the observation times of the process and the observed coordinates of the measure space can be random. In particular, at an observation time $t$, $X_t$ is observed at random points, and these observations may be incomplete.  
We do not assume to know the distribution of the stochastic process or the observation framework and will learn exclusively from the training data. We only require that our (fairly weak) \Cref{assumption:1,assumption:3,assumption:4,assumption:5,assumption:6,assumption:7,assumption:8} are satisfied. This allows us to train a fully data-driven model approximating the optimal online prediction of $X$, which is given by the conditional expectation process.
In the following we give precise definitions together with the needed assumptions to establish our theoretical guarantees, following the framework of \citet{krach2022optimal}.
For simplicity, we do not use the extensions discussed in \citet{NJODE3, krach2024learning,heiss2024nonparametric} in our description here. However, we remark that it is straightforward to include them in our considered framework, following the explanations in those papers.

\subsection{Technical Background and Observation Framework}
\label{sec:Technical Background and Mathematical Notation}

Let $ d_\Xi, d_X \in \N$ be dimensions and $T > 0$ be the fixed time horizon. We consider a bounded set $\Xi\subset \R^{d_\Xi}$ equipped with its Borel $\sigma$-algebra and a probability measure $\mu_\Xi$, forming the probability space $(\Xi, \mathcal{B}(\Xi), \mu_{\Xi})$. We additionally define a separable Hilbert space $\mathcal{H}:= L^2(\Xi, \R^{d_X})$ of $\mu_\Xi$-equivalence classes of measurable functions $u:\Xi\rightarrow \R^{d_X}$, where for $\varphi \in \mathcal{H}$ and with $|\cdot|_p$ the standard $p$-norm, 
\begin{equation*}
    \lVert \varphi \rVert_{\mathcal{H}}^2 := \int_\Xi |\varphi|_2^2 d\mu_\Xi < \infty.
\end{equation*}
Next, we consider a filtered probability space $(\Omega, \F, \mathbb{F} := \{\F_t\}_{0 \leq t \leq T}, \P )$ and define an $\mathbb{F}$-adapted c\`adl\`ag stochastic process \footnote{A stochastic process $X=(X_t)_{t\in\mathcal{T}}$ is a collection of random variables on a common probability space over an index set $\mathcal{T}$.}  taking values in $\mathcal{H}$, $X:[0, T] \times \Omega \rightarrow \mathcal{H}$, for which we assume square integrability. We also introduce a random variable $\xi:\Omega\rightarrow\Xi$ independent of $X$ with law $\mu_\Xi$, i.e.~$\mu_\Xi=\xi_\#\P$ is the pushforward measure. As a result, for $\omega \in \Omega$, the composition $X_{t,\xi} := (X_t)(\xi)\in L^2(\Omega,R^{d_X})$ is well-defined $\P\text{-a.s.}$ and depends only on the equivalence class $X_t(\omega)\in\mathcal{H}$. Although this can be constructed by considering the product probability space $(\Omega \times \Xi, \mathcal{F}\times\mathcal{B}(\Xi), \P\times\mu_\Xi)$, the formulation here remains in $(\Omega, \mathcal{F}, \P)$ by interpreting spatial evaluations as pullbacks of $\mu_\Xi$-equivalence classes through the random sample $\xi$. This allows terms to be expressed in terms of equivalence classes without making an explicit choice of a representative function.

\begin{rem}
    If $\Xi$ is a finite set, which we can assume to be $\{ 1, \dotsc, d\}$ without loss of generality, and $d_X=1$, then $X$ can equivalently be viewed as an $\R^d$-valued stochastic process, where $\xi \in \Xi$ denotes the coordinates. Moreover, if $|\Xi|=1$ then we are precisely in the setting of \citet{krach2022optimal}. Hence, our setting is a generalization of the one in \citet{krach2022optimal}. When $\Xi$ is infinite, for $\omega\in\Omega$, we can still think of $\xi(\omega) \in \Xi$ as a random point on which the equivalence class $X_t(\omega)$ is evaluated to yield $X_{t,\xi}(\omega)$ on the output space $\R^{d_X}$, independent of the chosen representative $\P$-almost surely. For $d_X > 1$, the corresponding values have $d_X$ components.
\end{rem}

The observation framework consists of the following ingredients.

\begin{itemize}
\item $n: \Omega \to \N_{\geq 0}$, an ${\F}$-measurable random variable, is the random number of observation times, including $t_0$. 
\item $K := \sup \left\{k \in \N \, | \, \P(n > k) > 0 \right\} \in \N \cup\{\infty\}$ is the maximum value of $n-1$, the number of additional observations other than at $t_0$.
\item  $t_k: \Omega \to [0,T] \cup \{ \infty \}$ for $0 \leq k \leq K$ are the \emph{sorted}\footnote{For all $\omega \in \Omega$, $0=t_0 < t_1(\omega) < \dotsb < t_{n(\omega)-1}(\omega) \leq T$.} stopping times\footnote{In particular, $t_i$ is a random variable such that $\{ t_i \leq t\} \in {\F}_t$ for all $1 \leq i \leq n$ and $t \in [0,T]$.}, describing the random observation times, with $t_k({\omega}) = \infty$ if and only if $n({\omega}) \le k$, and $t_0 = 0$.
\item $\tau : [0,T] \times \Omega \to [0,T], \, (t, \omega) \mapsto \tau(t, \omega) := \max\{ t_i(\omega) | 0 \leq i \leq n(\omega)-1, t_i(\omega) \leq t \}$ is the time of the last observation before a certain time $t$.
\item $\kappa : [0,T] \times \Omega \to \N, \, (t, \omega) \mapsto \kappa(t, \omega) := \max\{ i \in \N | 0 \leq i \leq n(\omega)-1, t_i(\omega) \leq t \} $ is the index of the random observation time before time $t$.
\item $J_k:  \Omega \to \N$ for $0\leq k \leq K$ are ${\mathcal{F}}_{t_k}$-measurable random variables, representing the random number of discrete observation points at observation time $t_k$.
\item $\xi^k_j :  \Omega \to \Xi $ for $0\leq k \leq K, j \in \N$ are ${\mathcal{F}}_{t_k}$-measurable random variables, representing the random observation points in space. 
\item $M^k_j :  \Omega \to \{ 0,1 \}^{d_X}$ for $0\leq k \leq K, j \in \N$ are ${\mathcal{F}}_{t_k}$-measurable random variables representing the observation masks on the output space. In particular, $M^k_j \odot X_{t_k, \xi^k_j}$ are the observed values\footnote{Here, $\odot$ is the element-wise multiplication (Hadamard product).} of the process $X$ at observation time $t_k$ and observation point $\xi^k_j$. 
\end{itemize}

Through the masked observations, we define the filtration of the currently available information as $\mathbb{A} := (\mathcal{A}_t)_{t \in [0,T]}$ with 
\begin{equation*}
\mathcal{A}_t := \boldsymbol{\sigma}\left( M^i_j \odot X_{t_i, \xi^i_j}, t_i, J_i, \xi^i_j, M^i_j, t_i \leq t, j \leq J_i \, \middle | \,  (i,j) \in \N^2   \right),
\end{equation*} 
where $\boldsymbol\sigma(\cdot)$ denotes the generated $\sigma$-algebra. In other words typical generators of $\mathcal{A}_t$ are determined through $ M^i_j \odot X_{t_i, \xi^i_j}, t_i, J_i, \xi^i_j, M^i_j$ and their intersection with $t_i \leq t$ and $ j \leq J_i$, for $(i,j) \in \mathbb{N}^2$. Notice the meaning of $J_i$, which determines the number of random observations via $\xi$.

Clearly $\mathcal{A}_t = \mathcal{A}_{\tau(t)}$ for all $t \in [0,T]$. Furthermore for any fixed observation (stopping) time $t_k$ the stopped and pre-stopped $\sigma$-algebras at $t_k$ satisfy \citep[Definitions~2.37 and~8.1]{KarandikarRao2018}
\begin{equation*}
\begin{split}
\mathcal{A}_{t_k} = \boldsymbol{\sigma}\left(  M^i_j \odot X_{t_i, \xi^i_j}, t_i, J_i, \xi^i_j, M^i_j \,\middle|\, (i,j) \in \N^2: i \leq k, j \leq J_i \right), \quad \mathcal{A}_{t_k-}  = \mathcal{A}_{t_{k-1}} \vee \boldsymbol{\sigma}\left(  t_k \right).
\end{split}
\end{equation*} 
The first assertion follows from the very definition of $\mathcal{A}_t $. Second, the pre-stopped sigma algebra $\mathcal{A}_{t_k-}$ contains all information about the generating stopping time $t_k$ but nothing about the events happening at $t_k$ which are not determined by $t_k$. At $t_{k}$ events of the type $M^i_j \odot X_{t_k, \xi^k_j}, J_k, \xi^k_j, M^i_j$ intersected with $ j \leq J_k $ are added, whence $\mathcal{A}_{t_k-}$ only has information stemming from $t_k$ in addition to the information of $\mathcal{A}_{t_{k-1}}$.

We are interested in learning the $L^2$-optimal prediction of $X_t$ given the currently available information $\mathcal{A}_t$. \Cref{lemma: L2 optimality of conditional expectation} implies that the optimizer is given by the equivalence class of the conditional expectation $\hat{X}_{t} = \E\left[ X_{t} \middle| \mathcal{A}_t \right]\in L^2(\Omega \times \Xi, \R^{d_X})$ for $t \in [0,T]$.
This conditional expectation is not trivial to compute, especially if only training data is available and the underlying distributions are not known, as in our setting. However, even in the case that the underlying distributions are known, it can be difficult, or impossible, to derive a closed form solution and expensive to compute numerical approximations. Therefore, our goal is to derive and train a model that approximates $\hat{X} = (\hat{X}_{t})_{t\in [0,T]}$ with theoretical guarantees.

Let $O_{i,j} := (M^i_j \odot X_{t_i, \xi^i_j}, t_i, \xi^i_j, J_i, M^i_j) \in \R^{d}$ for $d = d_X + 1 + d_\Xi+ 1 + d_X$ be the \emph{collection of information} gathered at observation time $t_i$ and coordinate $\xi^i_j$. The sequence of random variables $O_i := (O_{i,j})_{1\leq j \leq J_i}$ is then the collection of observations at time $t_i$, and $O_{[0,t]} := (O_0, \dotsc, O_{\kappa(t)}, 0, \dotsc) \in (\R^{d})^{\N}$ is the total information gathered up to time $t$\footnote{For the last two definitions, we use flattened versions of the respective tuples, by abuse of notation. In particular, instead of a tuple of tuples of elements in $\R^d$, we just consider its projection onto a tuple of elements in $\R^d$.}. Note that the vector is a (flattened) stacking of the finite collections of information with zeros trailing to infinity thereafter and accommodates any realized number of observations. Since $\boldsymbol\sigma(O_{[0,t]}) = \mathcal{A}_t$, we can assert the existence of a function of these observations that corresponds to the conditional expectation. The Doob-Dynkin lemma \citep[Lemma~1.14]{kallenberg2021foundations} implies that for every $s,t \in [0,T]$, there exists a measurable function $F_t:(R^d)^\N\rightarrow \mathcal{H}$ such that $F_t(O_{[0, s]}) = \E[X_t |\mathcal{A}_s]$ holds $\P$-almost surely. We then define
\begin{equation*}
    F : [0,T]  \times (\R^{d})^{\N} \to \mathcal{H},\;(t, O) \mapsto F(t,O):= F_{t}(O)
\end{equation*}
and assume joint measurability of $F$, noting that $\hat X_t = F(t, O_{[0, \tau (t)]})$ $\P$-almost surely, and that the composition $\hat X_{t,\xi}:=(\hat X_t)(\xi) = (F(t,O_{[0,\tau(t)]}))(\xi)$. We write $F(t, \xi, O_{[0, \tau(t)]}):=(F(t,O_{[0,\tau(t)]}))(\xi)$ for the composition by abuse of notation.
\\



\subsection{Model Assumptions}

We make modest assumptions about our observation framework and the function $F$ that recovers the conditional expectation process $\hat{X}_{t}$. 

\begin{assumption} \label{assumption:1} 
For every $0\leq k, m \leq K$, $(M^k_j)_{j\in\N}$ is independent of $t_m$, $n$, $J_m$, $(\xi^m_j)_{j\in\N}$ and 
$ \P ((M^k_j)_i =1 ) > 0$ for every component  $1 \leq i \leq d_X$ of the vector (every component can be observed at any observation time and point).
\end{assumption} 
\begin{assumption} \label{assumption:3}
Almost surely $X$ is not observed at a jump. That is, for all $ 1 \leq i \leq K$, $1\leq j \leq J_i$, $X_{t_i}=X_{t_{i-}}$ $\P\text{-a.s.}$

\end{assumption} 
\begin{assumption} \label{assumption:4}
We assume that there exists a measurable function $f: [0,T]\times(\R^d)^\N\rightarrow\mathcal{H}$ that is a generalized derivative of $F$ in the sense that for all $t\in[0, T]$, 
\begin{equation*}
    F(t, O_{[0, \tau(t)]})=F(\tau(t),O_{[0, \tau(t)]}) +\int_{\tau(t)}^tf(s,O_{[0,\tau(t)]})ds \quad \P\text{-a.s. in }\mathcal{H}.
\end{equation*}
Writing $f(t, \xi, O_{[0, \tau(t)]}):=(f(t, O_{[0, \tau(t)]}))(\xi)$ by the same abuse of notation, this yields
\begin{equation*}
    F(t,\xi,O_{[0,\tau(t)]}) = F(\tau(t),\xi,O_{[0,\tau(t)]}) + \int_{\tau(t)}^t f(s,\xi,O_{[0,\tau(t)]}) ds\quad\P\text{-a.s. in }L^2(\Omega,\R^{d_X}).
\end{equation*}
Moreover, taking $t_{-1} := t_{0-}$ as the initialization of the model, we assume that
\begin{equation}\label{equ:assumption4 bound}
\E\left[ \frac{1}{n} \sum_{i=0}^{n-1} \frac{1}{J_i} \sum_{j=1}^{J_i} \Big(|F(t_i, \xi^i_j , O_{[0,t_i]} )|_2^2 + |F(t_{i-1}, \xi^{i}_j , O_{[0,t_{i-1}]} )|_2^2 
+  \int_0^T | f(t, \xi^i_j , O_{[0,\tau(t)]} )  |_2^2 dt  \Big) \right] < \infty.
\end{equation}


\end{assumption} 
\begin{assumption} \label{assumption:5}
We assume square integrability over the observations $\displaystyle \E\left[\frac{1}{n} \sum_{i=0}^{n-1} \frac{1}{J_i} \sum_{j=1}^{J_i} |X_{t_i, \xi^i_j}|_2^2 \right] < \infty$.
\end{assumption} 
\begin{assumption} \label{assumption:6}
The random number of observation times $n$ and the maximum of the random number of observation points $\sup_{0\leq i < n} J_i$ are integrable, i.e., $\E[n] < \infty$ and $\E[\sup_{0\leq i < n} J_i]<\infty$.
\end{assumption} 
\begin{assumption} \label{assumption:7}
All observation points $\xi^i_j$, $0\leq i \leq K, 1 \leq j \leq \N$ are independent and identically distributed (i.i.d.) copies of $\xi\sim\mu_\Xi$.
\end{assumption} 
\begin{assumption} \label{assumption:8}
The process $X$ is independent of the observation framework, i.e., of the random variables $n, (t_k, J_k, \xi_j^k, M_j^k)_{k,j \in \N_{\ge0}}$. 
\end{assumption}

The following remarks clarify how our assumptions differ from earlier versions of the model, including equivalent formulations with weaker assumptions.

\begin{rem}\label{rem:extension assumptions M0}
In contrast to \citet{krach2022optimal}, we do not need an additional assumption on $M^0_j=1$, since we can directly work with the masked observations $M^i_j \odot X_{t_i, \xi^i_j}$. Furthermore, the previous assumption of continuity of $F$ and $f$ is weakened to integrability, as defined in (\ref{equ:assumption4 bound}).
\end{rem}
\begin{rem}
    \Cref{assumption:3} can be weakened by instead assuming that if $X$ is observed at a jump, then the entire jump itself is observed, i.e. both $X_{t_i-, \xi^i_j}$ and $X_{t_i, \xi^i_j}$ are observed. Using $X_{t_i-, \xi^i_j}$ for the $t_i-$ part of the loss function re-establishes the same results.
    In general, to do so, the model needs to receive both observations, $X_{t_i-, \xi^i_j}$ and $X_{t_i, \xi^i_j}$, as input. To correctly learn the jumps, the model has to get $X_{t_i, \xi^i_j}$ as input at the jump, noting that $X_{t_i-, \xi^i_j}$ does not need to be fed as input to the model if it does not provide additional information (for example, if there is no path-dependence). However, in the case that the left limit does provide additional information, it is necessary to provide the model with both inputs to lose no information. There are multiple natural ways to deal with this in the NJ-ODE framework. We suggest the following two: i) Replace the single $\rho$ evaluation with two consecutive jumps, first feeding $X_{t_i-, \xi^i_j}$ and then $X_{t_i, \xi^i_j}$ (without applying the neural ODE $f$ in between); ii) use $X_{t_i-, \xi^i_j}$ as another part of the input process, i.e., pass the observations $X_{t_i-, \xi^i_j}$ via additional coordinates of the input process, while using $X_{t_i, \xi^i_j}$ as the standard coordinates of the input process \citep[this formulates it in the input-output setting of][]{heiss2024nonparametric}.
\end{rem}
\begin{rem}
    The independence assumptions (\Cref{assumption:1,assumption:8}) can be replaced by conditional independence assumptions as in \citet{NJODE3}.
\end{rem}

As in \cite{heiss2024nonparametric}, the following distance is used throughout the article to compare the similarity of any two c\`adl\`ag processes at arbitrary observation time points $(t_k, \xi^k_1)$ with observation mask $M_{1}^k$. Ultimately, this will be the metric in which we define the convergence of the Operator NJ-ODE to the conditional expectation.

\begin{definition}\label{def:indistinguishability}
The (pseudo) metrics $d_k$ are defined for any two $\mathbb{A}$-adapted c\`adl\`ag processes $\eta, \zeta$ taking values in $\C$, as 
\begin{equation}\label{equ: pseudo metric}
    d_k (\eta, \zeta) = c_0(k)\,  \E\left[ \mathds{1}_{\{n > k\}} \left(  \left| M_{1}^k \odot ( \eta_{t_k,\xi^k_1} - \zeta_{t_k,\xi^k_1} ) \right|_1 + \left| \eta_{t_k-, \xi^k_1} - \zeta_{t_k-,\xi^k_1} \right|_1\right) \right],
\end{equation}
where $c_0(k) = (\P(n > k))^{-1}$.
We call the processes \emph{indistinguishable}, if $d_k(\eta, \zeta)=0$ holds for every $0 \leq k \leq K$.
\end{definition}

This metric represents exactly what we can learn from the training data that is provided through the observation framework. 
In particular, it is only possible to correctly approximate the target process at times and locations where we can potentially observe it. 
Since we are working with expectations and the observations occur at random times and points, all of their possible realizations are taken into account in the definition of the metric.
However, if there are certain time intervals or location sets where the probability of observation is zero, then it is impossible to learn the dynamics of the target process there, and the metric does not take those regions into account. 
While it is possible to remove the observation mask from the second term of the metric, it is needed for the on spot term, since the $\mathbb{A}$-adapted processes are, in general, not independent of $M_{1}^k$ at $t_k$ (which holds at $t_k-$). In particular, the on spot terms can only be controlled where they are actually observed, while the left-limit terms are always controlled\footnote{
    If the observation $t_{k+1}$ can happen arbitrarily close to the previous observation $t_k$, then this allows to always control the on spot terms \citep[cf.~][Theorem~5.1]{crowell2025neural}.
}.

\section{The Operator Neural Jump ODE Model}\label{sec:proposed method}

In the following, we combine the NJ-ODE model \citep{krach2022optimal} with ideas and parts of the neural operator framework of \citet{kovachki2023neural} such that it becomes applicable for predicting function-valued processes.
By its nature, the path signature is easily defined for paths taking values in a finite-dimensional space. Since this is not the case here, we shall not use the signature transform. In \citet{krach2022optimal} and follow-up works \citep{NJODE3,krach2024learning}, the signature was a key ingredient for deriving the theoretical guarantees. Here, we will instead use the recurrent structure of the model to derive equivalent results. 

We start by briefly discussing the main challenges in this infinite-dimensional setting.

\paragraph{Arbitrary number of concurrent observation points.} 
To be generally applicable, the model needs to be able to process an arbitrary number of observations as concurrent inputs. Since their spatial ordering through the observation points $\xi_j^k$, $1 \leq j \leq J_k$ does not naturally translate into hierarchical ordering, the way these inputs are processed should be order independent. 
We achieve this utilizing \emph{Integral Kernel Operators} from the Neural Operator framework of \citet{kovachki2023neural}.

\paragraph{Functional-valued model.}
Our objective is to optimally predict, for every time $t$, the random process $X_t$ taking values in the Hilbert space $\C$. The target of our model is therefore the $\mu_\Xi$-equivalence class of the conditional expectation. The model, however, must be evaluated point-wise, so its prediction is necessarily a representative that induces an element in $\C$.
Again we can borrow the approach from the Neural Operator framework, which simply uses the evaluation point $\xi \in \Xi$ as additional input to the model. Hence, one only needs to ensure integrability (for example through bounded-output neural networks) of the mapping from $\xi$ to the model output to ensure that this representative induces a well-defined element of $\C$.
There are two possible approaches for using the evaluation point $\xi$ as input to the model. The standard approach (also used in \citep{kovachki2023neural}) is to input it together with the observations. Then, for example, a kernel can be used to link the evaluation point with the observation points. The disadvantage of this approach is that the entire model has to be reevaluated from $t=0$ to $t=T$ for each evaluation point $\xi$. A more efficient approach is to pass the evaluation point $\xi$ only to the readout network as input. Hence, the model's latent variable $H_t$ needs to simultaneously carry the information for all possible evaluation points, and only through the readout map is this projected to a specific evaluation point $\xi$. 
Therefore, the computationally intensive part to compute $H_t$ only needs to be done once. However, this comes at the price that, in general, $H_t$ has to be much higher-dimensional to contain all the needed information. Moreover, this does not allow for efficient linking between observation points and the evaluation point $\xi$.
In the following, we present the model using the standard Neural Operator type framework, where the entire model needs to be reevaluated for each evaluation point $\xi$, keeping in mind that changes to this framework are possible to increase computational efficiency.

In the following we assume that ${\mathcal{N}}$ is a set of feedforward neural networks with Lipschitz continuous activation functions such that for any $d,D \in \N$, any finite measure $\mu$ on  $\R^d$  and any $1 \leq p < \infty$  we have that  ${\mathcal{N}}$ is dense in the space $L^p( \R^d,\R^D;\mu)$ (with respect to the $L^p$-norm). In particular, ${\mathcal{N}}$ has to satisfy the standard $L^p$ universal approximation theorem, which is the case e.g.\ for the set of 1-hidden-layer neural networks with continuous, bounded and non-constant activation function \cite[Theorem~1]{hornik1991approximation}. Moreover, we assume that all affine functions are in $\mathcal{N}$. 

\subsection{O-NJ-ODE Definition}

We define the pure-jump process counting the additional observation times
$n_t:=\sum_{i=1}^{n-1} 1_{\left[t_i, \infty\right)}(t)$, $t \in [0,T]$, whose increments $dn_t$ vanish everywhere except at observation times $t_k$ where $\Delta n_{t_k} = 1$. 
With this, our Operator Neural Jump ODE model is defined as follows.

\begin{definition}\label{def:O-NJ-ODE}
    The  \emph{Operator Neural Jump ODE (O-NJ-ODE)} is defined as the following operator that can be evaluated for any $\xi \in \Xi$,
    \begin{equation}\label{equ:O-NJ-ODE}
    \begin{split}
    H_0(\xi) &= \rho_{\theta_2}\left( H_{0-}(\xi), 0,  \xi,  \psi_{\theta_3}(0, \xi,O_{0}) \right), \\
    dH_t (\xi) &= f_{\theta_1}\left(H_{t-}(\xi), t, \tau(t), \xi \right) dt +\left[ \rho_{\theta_2}\left( H_{t-}(\xi), t, \xi,  \psi_{\theta_3}(t,\xi,O_{\kappa(t)})  \right) - H_{t-}(\xi) \right] dn_t, \\
    Y_t(\xi) &=  g_{ \theta_4}\left(H_t(\xi)\right),
    \end{split}
    \end{equation}
    where we use the self-imputed observations in the inputs $O_{\kappa(t)}$, defined as
    \begin{equation}\label{equ:self-imputed observations}
        \tilde{X}_{t_k}^j = M_j^k \odot X_{t_k, \xi_{j}^{k}} + (1_{d_X} - M_j^k) \odot Y_{t_k-}\left(\xi_{j}^{k}\right)
    \end{equation}
    and the prior prediction $Y_{0-}(\xi) = g_{ \theta_4}\left(H_{0-}(\xi)\right)$ at $t=0$, with $H_{0-}(\xi) = \rho_{\theta_2}\left( 0, 0,  \xi,  0  \right)$. We call $\psi_{\theta_3}$ the generalized kernel, which has a special structure; see below.
    The functions $f_{\theta_1}, \rho_{\theta_2}, \psi_{\theta_3},  g_{\theta_4} \in \mathcal{N}$ are feedforward neural networks with trainable parameters $\theta = (\theta_1, \theta_2,  \theta_3, \theta_4) \in \Theta$, where $\Theta$ is the set of all possible weights for the model. The algorithm for the O-NJ-ODE is displayed in \Cref{algo: ONJODE}.
\end{definition}

\citet[Thm. 7, Chap. V]{Pro1992} implies that a unique solution of \eqref{equ:O-NJ-ODE} exists for every $\xi \in \Xi$. We write $Y_{t,\xi}^\theta(O_{[0,t]})$ to emphasize the dependence of the output $Y_{t,\xi}=Y_t(\xi)$ on $\theta$ and the collection of information $O_{[0,t]}$.

\paragraph{The Generalized Kernel}
The function $\psi_{\theta_3}(t,\xi,O_{\kappa(t)})$, which we denote as the \emph{generalized kernel}, is the critical ingredient that learns spatial relations at each observation time. One subtlety that must be considered is that the inputs $O_i$ for $i=0,\ldots,n-1$ are of unbounded dimension; $\psi_{\theta_3}$ must encode each of these observation collections $O_{\kappa(t)}$, which may have variable dimension, into a fixed-dimensional object to be used as input for the jump neural network $\rho_{\theta_2}$. Consequently, the O-NJ-ODE can use an implicit projection of the new inputs onto a finite (but trainable) dimension before passing it to a standard neural network $\tilde \psi_{\theta_3}$. Alternatively, the generalized kernel can aggregate fixed and finite-dimensional information locally in space. This has the added benefit of including spatial context for enhanced performance. We consider a generalized kernel structure of the form
\begin{equation} \label{eq:gen kernel}
    \psi_{\theta_3}(t,\xi,O_{\kappa(t)})=\frac{1}{J_{\kappa(t)}} \sum_{j=1}^{J_{\kappa(t)}} \varphi^1_{\theta_3}\left(\xi, \xi_{j}^{\kappa(t)}, \tilde{X}_{t}^j, M_j^{\kappa(t)}, J_{\kappa(t)} \right) \odot \varphi^2_{\theta_3}\left(j\right),
\end{equation}
where the local aggregation of all observations at time $t$ is done without prior information loss through projection. In making this choice of generalized kernel, $\varphi^1_{\theta_3}, \varphi^2_{\theta_3} \in \mathcal{N}$ with trainable parameters $\theta_3 \in \Theta$.

With sufficient modeling capacity and considering UAT assumptions, we can learn the global spatial relationship to enrich our hidden state update. Many special architectures with different advantages and disadvantages can be chosen for the generalized kernel. In our proof, we will use a general and especially simple architecture, however, in practice others might be preferred. We refer to \citet[Section~4]{kovachki2023neural} for a non-exhaustive overview of possible choices. Although direct inclusion of spatial context is theoretically unnecessary, its practical inclusion will improve learning and performance. 

\begin{rem}
    For a more efficient model, we use the self-imputed observations $\tilde{X}_{t_k}^j$ as inputs instead of the masked observation $M_j^k \odot X_{t_k, \xi_{j}^{k}}$. At $t=0$, the self-imputation could in principle be done with $0$. Instead, we first apply the jump part $\rho_{\theta_2}$ of the model without any input to generate the uninformed prior mean prediction $H_{0-}$. To produce good results, we also include the initial time $t_0$ in the objective function (\ref{equ:Psi}).
\end{rem}

\begin{algorithm}[H]
\begin{algorithmic}[1]
\State\textbf{Input: }Data points $\{O_i\}_{i=0}^{n-1}$ across observation times, with $ O_i =(O_{i,j})_{j=1}^{J_i}$ and $O_{i,j} = (\tilde X_{t_i}^j, t_i, \xi_j^i, J_i, M_j^i)$
\State \textbf{Evaluation point:} $\xi\in\Xi$ \Comment{The O-NJ-ODE is trained at one evaluation point}
\State \textbf{Initialization: } For time $t=t_{0-}$, consider $H_{0-}(\xi)=\rho_{\theta_2}(0, 0, \xi, 0)$ and $Y_{0-}(\xi)=g_{\theta_4}(H_{0-}(\xi))$
\State $t_{n} :=T$
\State $Y = [\,]$, $\text{times}=[\,]$ \Comment{Empty lists to collect output values}
\For{$i=0$ to $n-1$} 
\State $H_{t_i}(\xi) = \rho_{\theta_2}(H_{t_{i}-}(\xi), t_i, \xi, \psi_{\theta_3}(t_i, \xi, O_{i}))$ \Comment{Jump update of hidden state at obs. time}
\State $Y_{t_i}(\xi) = g_{\theta_4}(H_{t_i}(\xi))$ \Comment{Decoded process at obs. time}
\State add $Y_{t_i}(\xi)$ to the list $Y$, add $t_i$ to the list $\text{times}$
\State $t=t_i$
\While{$t +\Delta t \le t_{i+1}$} 
\State $H_{t+\Delta t}(\xi) = H_t(\xi) + f_{\theta_1}(H_{t}(\xi), t+\Delta t, t, \xi) \, \Delta t$ \Comment{Hidden state evolves between observations}
\State $Y_{t+\Delta t}(\xi) = g_{\theta_4}(H_{t+\Delta t}(\xi))$
\State add $Y_{t+\Delta t}(\xi)$ to the list $Y$, add $t+\Delta t$ to the list $\text{times}$
\State $t=t+\Delta t$
\EndWhile
\EndFor
\State \textbf{return:} $Y$, $\text{times}$
\end{algorithmic}
    \caption{The O-NJ-ODE}
    \label{algo: ONJODE}
\end{algorithm}

\subsection{Objective Function}\label{sec:Objective Function}
Let $\mathbb{D}$ be the set of all c\`adl\`ag $\mathcal{H}$-valued $\mathbb{A}$-adapted processes.
By \Cref{assumption:4}, $\hat{X} \in \mathbb{D}$. $\mathbb{D}$ acts as the  set of candidates for approximating $X$ with the known information ($\mathbb{A}$-adaptedness).
For any $Z \in \mathbb{D}$ we use the same notations as introduced for $X$. 
Then we define our objective functions as
\begin{align} 
\Psi: \, \mathbb{D} &\to \R, &
Z &\mapsto \Psi(Z) := \E\left[ \frac{1}{n} \sum_{i=0}^{n-1} \frac{1}{J_i} \sum_{j=1}^{J_i}    \left\lvert M^i_j \odot ( X_{t_i, \xi_j^i} - Z_{t_i, \xi_j^i} ) \right\rvert_2^2 + \left\lvert M^i_j \odot (X_{t_i, \xi_j^i} - Z_{t_{i}-, \xi_j^i} ) \right\rvert_2^2 \right] \label{equ:Psi} \\
\Phi : \, \Theta &\to \R,  & \theta &\mapsto \Phi(\theta) := \Psi(Y_{\cdot}^{\theta}(O_{[0,\cdot]})). \label{equ:Phi}
\end{align}
$\Phi$ is our (theoretical) loss function, which is well-defined since the definition of $Y^\theta$ implies that its induced equivalence class is an element of $\mathbb{D}$. 

To construct the (practically relevant) empirical objective function, we formalize the sampling framework for observed data. We assume that we have a sample size of $N \in \N$ independent realizations of the path $X$ together with independent realizations of the observation framework. In particular, let us assume that for $1 \leq l \leq N$ and for all $k, j \in \N$ the random processes and variables $X^{(l)} \sim X$, $(M_j^k)^{(l)} \sim M_j^k$, $J_k^{(l)} \sim J_k$, $(\xi_j^k)^{(l)} \sim \xi_j^k$ and $(n^{(l)}, t_1^{(l)}, \dotsc, t_{n^{(l)}-1}^{(l)}) \sim ( n, t_1, \dotsc, t_{n-1})$ are respectively independent and identically distributed. Our training data set is assumed to be a realization of all $N$ paths, each of which is denoted by $O^{(l)}=(O^{(l)}_0, ..., O^{(l)}_{n^{(l)}-1})$ where $O^{(l)}_{i}=(O^{(l)}_{i,1}, ..., O^{(l)}_{i,J_i})$ for $i=0,...,n^{(l)}-1$ with $O^{(l)}_{i,j} := ((M_j^i)^{(l)}\odot X_{t_i^{(l)}, (\xi_j^i)^{(l)}}^{(l)}, t_i^{(l)},(\xi^i_j)^{(l)}, J_i^{(l)},(M^i_j)^{(l)})\in\R^d$. Here, $O_{i,j}^{(l)}$ is the realized information collection at observation coordinate $(t_i,\xi_j^i)$ of the sample path indexed by $l$. 
We write $(M_j^k)^{(l)}:=M_j^k(O^{(l)})$, $(Y^{\theta}_{t_k^{(l)},(\xi_j^k)^{(l)}})^{(l)} := Y^{\theta }_{t_k,\xi_j^k}\left(O^{(l)}\right)$, and $X_{t_k^{(l)}, (\xi_j^k)^{(l)}}^{(l)}:=X_{t_k,\xi_j^k}(O^{(l)})$ to simplify the notation. 

With this construction, we define our empirical objective function as the Monte Carlo estimate of the true objective function (\ref{equ:Phi}). For our $N$ path dataset, this is given by 
\begin{multline}\label{equ:appr loss function}
\hat\Phi_N(\theta) := \frac{1}{N} \sum_{l=1}^N  \frac{1}{n^{(l)}}\sum_{i=0}^{n^{(l)}-1} \frac{1}{J_i^{(l)}} \sum_{j=1}^{J_i^{(l)}} \left(  \left\lvert \left[M_j^i \odot \left( X_{t_i,\xi_j^i} - Y^{\theta }_{t_i,\xi_j^i}\right)\right](O^{(l)}) \right\rvert_2^2 \right. \\
+ \left. \left\lvert \left[M_j^i \odot \left( X_{t_i,\xi_j^i}- Y^{\theta }_{t_i-,\xi_j^i} \right)\right](O^{(l)}) \right\rvert_2^2 \right),
\end{multline}
and converges to the true objective function as $N\to\infty$ by \Cref{thm:MC convergence Yt}. This makes it a suitable loss function for training the O-NJ-ODE.

\begin{rem}
    We use the input-output loss function first introduced for Neural Jump ODEs in \citet{heiss2024nonparametric}. This loss function directly allows for an input-output setting and simplifies the proofs. 
    However, the standard loss function of \citet{krach2022optimal} could be used equivalently in the present setting (of coinciding input and output processes). Despite the slightly more complicated proof, the standard loss has an advantage in terms of the speed of convergence as discussed in \citet[Section~7]{heiss2024nonparametric}.
\end{rem}


\section{Convergence Guarantees}\label{sec:Convergence Guarantees}
In this section, we show that our fully data-driven Operator Neural Jump ODE converges to the optimal prediction.
Following \citet{krach2022optimal}, we first show convergence for the theoretical objective function $\Phi$. Afterwards, we prove that its Monte Carlo approximation $\hat \Phi$ converges to it.

\subsection{Convergence of Theoretical Loss Function}\label{sec:Convergence of Theoretical loss function}

Let ${\Theta}_m \subset \Theta$ be the compact set of possible weights $\theta = (\theta_1, \theta_2,  \theta_3, \theta_4) \in \Theta$  for the four neural networks defining the Operator NJ-ODE \eqref{equ:O-NJ-ODE}, such that their widths and depths are at most $m$  and all their norms satisfy $|\theta_i|_2 \leq m$. Note that $\argmin_{\theta\in\Theta_m}\Phi(\theta)\ne \emptyset$, since an application of Fatou's lemma shows that $\Phi(\theta)$ is lower semicontinuous on the compact set $\Theta_m$. When speaking of the projection of this set to one of the weights $\theta_i$, we use the notation ${\Theta}_m^i$

\begin{theorem}\label{thm:1} 
Let $\theta^{\min}_m \in \Theta_m^{\min} := \argmin _{\theta \in \Theta_m}\{ \Phi(\theta) \}$ for every $m \in \N$. If Assumptions~\ref{assumption:1} to \ref{assumption:8} are satisfied, then, for $m \to \infty$, the value of the loss function $\Phi$ \eqref{equ:Phi} converges to the minimal value of $\Psi$ \eqref{equ:Psi} which is uniquely achieved by $\hat{X}$ up to indistinguishability (cf. Definition~\ref{def:indistinguishability}), i.e.,
\begin{equation*}
\Phi(\theta_m^{\min}) \xrightarrow{m \to \infty} \min_{Z \in \mathbb{D}} \Psi(Z) = \Psi(\hat{X}).
\end{equation*}
Furthermore, for every $0 \leq k \leq K$ we have that $Y^{\theta_m^{\min}}$ converges to $\hat{X}$ in the metric $d_k$ \eqref{equ: pseudo metric} as $m \to \infty$.
\end{theorem}

Before proving the theorem, we require the following two lemmas from \citet{heiss2024nonparametric} reformulated for the current setting. The first follows straightforwardly from \cref{assumption:1}, \cref{assumption:8}, and the $L^2$-optimality of the conditional expectation. A full proof is provided for the second Lemma. 

\begin{lemma}\label{lemma: L2 optimality of conditional expectation}
    For an $\mathbb{A}$-adapted process $\eta$ 
    we have

    \begin{multline*}
        \E\left[\frac{1}{n}\sum_{i=0}^{n-1} \frac{1}{J_i} \sum_{j=1}^{J_i} \left\lvert M_j^{i} \odot (X_{t_i, \xi_j^i} - \eta_{t_i-,\xi_j^i}) \right \rvert_2^2\right] \\
        = \E\left[\frac{1}{n}\sum_{i=0}^{n-1} \frac{1}{J_i} \sum_{j=1}^{J_i} \left\lvert M_j^{i} \odot (X_{t_i, \xi_j^i} - \hat{X}_{t_i-,\xi_j^i}) \right \rvert_2^2\right] + \E\left[\frac{1}{n}\sum_{i=0}^{n-1} \frac{1}{J_i} \sum_{j=1}^{J_i} \left\lvert M_j^{i} \odot (\hat{X}_{t_i-, \xi_j^i} - \eta_{t_i-,\xi_j^i}) \right \rvert_2^2\right],
    \end{multline*}

    and also

    \begin{multline*}
        \E\left[\frac{1}{n}\sum_{i=0}^{n-1} \frac{1}{J_i} \sum_{j=1}^{J_i} \left\lvert M_j^{i} \odot (X_{t_i, \xi_j^i} - \eta_{t_i,\xi_j^i}) \right \rvert_2^2\right] \\
        = \E\left[\frac{1}{n}\sum_{i=0}^{n-1} \frac{1}{J_i} \sum_{j=1}^{J_i} \left\lvert M_j^{i} \odot (X_{t_i, \xi_j^i} - \hat{X}_{t_i,\xi_j^i}) \right \rvert_2^2\right] + \E\left[\frac{1}{n}\sum_{i=0}^{n-1} \frac{1}{J_i} \sum_{j=1}^{J_i} \left\lvert M_j^{i} \odot (\hat{X}_{t_i, \xi_j^i} - \eta_{t_i,\xi_j^i}) \right \rvert_2^2\right], \\
    \end{multline*}

    which, when taking the minimum over $\eta\in\mathbb{D}$, implies that both

    \begin{align*}
        \min_{\eta\in\mathbb{D}}\E\left[\frac{1}{n}\sum_{i=0}^{n-1} \frac{1}{J_i} \sum_{j=1}^{J_i} \left\lvert M_j^{i} \odot (X_{t_i, \xi_j^i} - \eta_{t_i-,\xi_j^i}) \right \rvert_2^2\right] &= \E\left[\frac{1}{n}\sum_{i=0}^{n-1} \frac{1}{J_i} \sum_{j=1}^{J_i} \left\lvert M_j^{i} \odot (X_{t_i, \xi_j^i} - \hat{X}_{t_i-,\xi_j^i}) \right \rvert_2^2\right] \\
        \min_{\eta\in\mathbb{D}}\E\left[\frac{1}{n}\sum_{i=0}^{n-1} \frac{1}{J_i} \sum_{j=1}^{J_i} \left\lvert M_j^{i} \odot (X_{t_i, \xi_j^i} - \eta_{t_i,\xi_j^i}) \right \rvert_2^2\right] &= \E\left[\frac{1}{n}\sum_{i=0}^{n-1} \frac{1}{J_i} \sum_{j=1}^{J_i} \left\lvert M_j^{i} \odot (X_{t_i, \xi_j^i} - \hat{X}_{t_i,\xi_j^i}) \right \rvert_2^2\right].
    \end{align*}
\end{lemma}

\begin{lemma}\label{lemma: expectation weighted sum over t_i terms}
    Let $U\sim U(0,1)$ be a uniform random variable with corresponding probability measure $\mu_U$ and consider a measurable function $\varphi:[0,T]\times\Omega\rightarrow\mathbb{R}$, such that for all $t \in[0,T]$ $\varphi(t)$ is $\mathcal{F}$-measurable. Then, if $U$ is independent of $\mathcal{F}$, we have for $\bar t :=\sum_{i=0}^{n-1}\mathds{1}_{\left(\frac{i}{2n}, \frac{i+1}{2n}\right]}\left(U\right)\,t_i + \mathds{1}_{\left(\frac{n+i}{2n}, \frac{n+i+1}{2n}\right]}\left(U\right)\,t_{i-1}$ that

    \begin{equation}
        \E_{\mu_U\times \P}\left[\varphi(\bar t)\right] = \frac{1}{2}\left(\E_\P\left[\frac{1}{n}\sum_{i=0}^{n-1}\varphi(t_i)\right] +\E_\P\left[\frac{1}{n}\sum_{i=0}^{n-1}\varphi(t_{i-1})\right]\right)
    \end{equation}

    where $\mu_U\times\P$ is the product measure.
\end{lemma}
\begin{proof}
    \begin{align*}
        \E_{\mu_U\times \P}\left[\varphi(\bar t)\right] &= \E_\P\left[\int_0^1\varphi(\sum_{i=0}^{n-1}\mathds{1}_{\left(\frac{i}{2n}, \frac{i+1}{2n}\right]}\left(u\right)\,t_i + \mathds{1}_{\left(\frac{n+i}{2n}, \frac{n+1+i}{2n}\right]}\left(u\right)\,t_{i-1})du\right] \\
        &=\E_\P\left[\sum_{i=0}^{n-1}\int_{\frac{i}{2n}}^{\frac{i+1}{2n}}\varphi(t_i)du + \int_{\frac{n+i}{2n}}^{\frac{n+1+i}{2n}}\varphi(t_{i-1})du\right] \\
        &=\E_\P\left[\sum_{i=0}^{n-1}\frac{1}{2n}\left(\varphi(t_i) +\varphi(t_{i-1})\right)\right] \\
        &=\frac{1}{2}\left(\E_\P\left[\frac{1}{n}\sum_{i=0}^{n-1}\varphi(t_i)\right] +\E_\P\left[\frac{1}{n}\sum_{i=0}^{n-1}\varphi(t_{i-1})\right]\right)
    \end{align*}
by Fubini's theorem.
\end{proof}

Now, we may prove the main convergence result, which proceeds in three steps. First, we establish that the conditional expectation is the unique minimizer of the objective function, up to indistinguishability. Second, we show that the O-NJ-ODE can approximate the conditional expectation by universal approximation arguments. Finally, we prove that this convergence holds with respect to the metric $d_k$.

\begin{proof}[Proof of Theorem~\ref{thm:1}]
\textbf{Step 1:} We first show that $\hat{X} \in \mathbb{D}$ is the unique minimizer of $\Psi$ up to indistinguishability as defined in Definition~\ref{def:indistinguishability}. Therefore, by \Cref{lemma: L2 optimality of conditional expectation},
\begin{equation*}
\begin{split}
\Psi(\hat{X}) 
	&= \E\left[\frac{1}{n}\sum_{i=0}^{n-1} \frac{1}{J_i} \sum_{j=1}^{J_i} \left\lvert M_j^{i} \odot (X_{t_i, \xi_j^i} - \hat{X}_{t_i, \xi_j^i} ) \right\rvert_2^2 + \left\lvert M_j^{i} \odot (X_{t_i, \xi_j^i} - \hat{X}_{t_i-, \xi_j^i} ) \right\rvert_2^2\right] \\
	&= \min_{\eta \in \mathbb{D}} \E\left[\frac{1}{n}\sum_{i=0}^{n-1} \frac{1}{J_i} \sum_{j=1}^{J_i} \left\lvert M_j^{i} \odot (X_{t_i, \xi_j^i} - \eta_{t_i,\xi_j^i}) \right \rvert_2^2\right] + \min_{\eta \in \mathbb{D}} \E\left[\frac{1}{n}\sum_{i=0}^{n-1} \frac{1}{J_i} \sum_{j=1}^{J_i} \left\lvert M_j^{i} \odot (X_{t_i, \xi_j^i} - \eta_{t_i-,\xi_j^i}) \right \rvert_2^2\right] \\
	&\leq \min_{\eta \in \mathbb{D}} \E\left[\frac{1}{n}\sum_{i=0}^{n-1} \frac{1}{J_i} \sum_{j=1}^{J_i} \left( \left\lvert M_j^{i} \odot (X_{t_i, \xi_j^i} - \eta_{t_i,\xi_j^i}) \right \rvert_2^2 +\left\lvert M_j^{i} \odot (X_{t_i, \xi_j^i} - \eta_{t_i-,\xi_j^i}) \right \rvert_2^2 \right)\right] \\
	&= \min_{\eta \in \mathbb{D}}\Psi(\eta),
\end{split}
\end{equation*}
Hence, $\hat{X}$ is a minimizer of $\Psi$, but it remains to show that it is unique up to indistinguishability as defined by the pseudo metric $d_k$ (\ref{equ: pseudo metric}). For a non-indistinguishable process $\eta \in \mathbb{D}$, there exists some $0 \leq k \leq K$ such that $d_k(\hat{X}, \eta) > 0$ and we have

\begin{align}
\Psi(\eta) &= \E\left[\frac{1}{n}\sum_{i=0}^{n-1} \frac{1}{J_i} \sum_{j=1}^{J_i} \left\lvert M_j^{i} \odot (X_{t_i, \xi_j^i} - \eta_{t_i, \xi_j^i} ) \right\rvert_2^2 + \left\lvert M_j^{i} \odot (X_{t_i, \xi_j^i} - \eta_{t_i-, \xi_j^i} ) \right\rvert_2^2\right] \notag\\
	&= \E\left[\frac{1}{n}\sum_{i=0}^{n-1} \frac{1}{J_i} \sum_{j=1}^{J_i} \left\lvert M_j^{i} \odot (X_{t_i, \xi_j^i} - \eta_{t_i, \xi_j^i} ) \right\rvert_2^2\right] + \E\left[\frac{1}{n}\sum_{i=0}^{n-1} \frac{1}{J_i} \sum_{j=1}^{J_i} \left\lvert M_j^{i} \odot (X_{t_i, \xi_j^i} - \eta_{t_i-, \xi_j^i} ) \right\rvert_2^2\right] \notag\\
	&= \E\left[\frac{1}{n}\sum_{i=0}^{n-1} \frac{1}{J_i} \sum_{j=1}^{J_i} \left\lvert M_j^{i} \odot (X_{t_i, \xi_j^i} - \hat{X}_{t_i,\xi_j^i}) \right \rvert_2^2\right] + \E\left[\frac{1}{n}\sum_{i=0}^{n-1} \frac{1}{J_i} \sum_{j=1}^{J_i} \left\lvert M_j^{i} \odot (\hat{X}_{t_i, \xi_j^i} - \eta_{t_i,\xi_j^i}) \right \rvert_2^2\right]  \notag\\ 
    &+ \, \E\left[\frac{1}{n}\sum_{i=0}^{n-1} \frac{1}{J_i} \sum_{j=1}^{J_i} \left\lvert M_j^{i} \odot (X_{t_i, \xi_j^i} - \hat{X}_{t_i-,\xi_j^i}) \right \rvert_2^2\right] + \E\left[\frac{1}{n}\sum_{i=0}^{n-1} \frac{1}{J_i} \sum_{j=1}^{J_i} \left\lvert M_j^{i} \odot (\hat{X}_{t_i-, \xi_j^i} - \eta_{t_i-,\xi_j^i}) \right \rvert_2^2\right] \notag\\
	&= \Psi(\hat{X})  + \E\left[\frac{1}{n}\sum_{i=0}^{n-1} \frac{1}{J_i} \sum_{j=1}^{J_i} \left\lvert M_j^{i} \odot (\hat{X}_{t_i, \xi_j^i} - \eta_{t_i, \xi_j^i} ) \right\rvert_2^2 + \left\lvert M_j^{i} \odot (\hat{X}_{t_i-, \xi_j^i} - \eta_{t_i-, \xi_j^i} ) \right\rvert_2^2\right] , \label{equ:ST}
\end{align}

where we have used again \Cref{lemma: L2 optimality of conditional expectation} in the third line. Hence, it is enough to show that the second term is greater than $0$, which requires three additional results. These will be applied to bound the second term from below by the distance $d_k(\hat{X},\eta)$.

First, let $c_1 := \E\left[ n \right]^{1/2} \in (0, \infty)$ by \cref{assumption:6}, then the Hölder inequality, together with the fact that $n \geq 1$, yields 
\begin{equation}\label{equ:HI}
\E\left[ \left\lvert \eta \right\rvert_2 \right] 
	= \E\left[ \frac{\sqrt{n}}{\sqrt{n}} \left\lvert \eta \right\rvert_2 \right] 
	\leq c_1 \, \E\left[ \frac{1}{n} \left\lvert \eta \right\rvert_2^2 \right]^{1/2}.
\end{equation}

Equivalence of the $1$-norm and $2$-norm implies that there exist constants $c_{2,1}, c_{2,2} > 0$ such that

\begin{equation}\label{equ:EN}
\begin{split}
 \E\left[ \mathds{1}_{\{n > k\}} \left\lvert M_j^k \odot \left(\hat{X}_{t_k, \xi_j^k} - \eta_{t_k,\xi_j^k}\right)  \right\rvert_2 \right] 
  &\geq c_{2,1} \E\left[ \mathds{1}_{\{n > k\}} \left\lvert M_j^k \odot \left(\hat{X}_{t_k, \xi_j^k} - \eta_{t_k,\xi_j^k}\right)  \right\rvert_1 \right] \\
 \E\left[ \mathds{1}_{\{n > k\}} \left\lvert M_j^k \odot \left(\hat{X}_{t_k-, \xi_j^k} - \eta_{t_k-,\xi_j^k}\right)  \right\rvert_2 \right] &\geq  c_{2,2} \E\left[ \mathds{1}_{\{n > k\}} \left\lvert M_j^k \odot \left(\hat{X}_{t_k-, \xi_j^k} - \eta_{t_k-,\xi_j^k}\right)  \right\rvert_1 \right]
 \end{split}
\end{equation}

where we let $c_2 = \min(c_{2,1}, c_{2,2})$.

By Assumption~\ref{assumption:1} we know that $c_3(k) := \min_{1 \leq i \leq d_X} \P ((M_{j}^k)_i = 1) > 0$ and that $(M_{j}^k)_i$ is independent from $t_k$, $n$, $J_k$, $\xi_j^k$ and $\mathcal{A}_{t_k-}$. Hence, we have for any $0 \leq k \leq K$ that

\begin{multline}\label{eq:removing M}
\E\left[ \mathds{1}_{\{n > k\}} \left\lvert M_j^k \odot ( \hat{X}_{t_k-,\xi_j^k} - \eta_{t_k-,\xi_j^k} ) \right\rvert_1 \right] 
    	= \E\left[ \mathds{1}_{\{n > k\}} \sum_{i=1}^{d_X} (M_j^k)_i \left\lvert  (\hat{X}_{t_k-,\xi_j^k} - \eta_{t_k-,\xi^k_j})_i  \right\rvert \right] \\
	=\sum_{i=1}^{d_X} \E\left[(M_j^k)_i \right] \, \E\left[\mathds{1}_{\{n > k\}}  \left\lvert  (\hat{X}_{t_k-,\xi_j^k} - \eta_{t_k-,\xi_j^k})_i  \right\rvert \right] \\
	\geq c_3(k) \, \E\left[ \mathds{1}_{\{n > k\}} \left\lvert \hat{X}_{t_k-,\xi_j^k} - \eta_{t_k-,\xi_j^k}  \right\rvert_1 \right] ,
\end{multline}

We then sequentially apply the above results to bound the second term of \eqref{equ:ST} from below by a positive quantity,
\begin{align}
 \Psi  (\eta)-\Psi(\hat{X}) &= \E\left[\frac{1}{n}\sum_{i=0}^{n-1} \frac{1}{J_i} \sum_{j=1}^{J_i} \left\lvert M_j^{i} \odot (\hat{X}_{t_i, \xi_j^i} - \eta_{t_i, \xi_j^i} ) \right\rvert_2^2 + \left\lvert M_j^{i} \odot (\hat{X}_{t_i-, \xi_j^i} - \eta_{t_i-, \xi_j^i} ) \right\rvert_2^2\right] \notag \\
	&= \E\left[\frac{1}{n}\sum_{i=0}^K \mathds{1}_{\{n > i\}}  \frac{1}{J_i} \sum_{j=1}^{J_i} \left\lvert M_j^{i} \odot (\hat{X}_{t_i, \xi_j^i} - \eta_{t_i, \xi_j^i} ) \right\rvert_2^2 + \left\lvert M_j^{i} \odot (\hat{X}_{t_i-, \xi_j^i} - \eta_{t_i-, \xi_j^i} ) \right\rvert_2^2\right] \notag \\
	&\geq \E\left[\frac{1}{n} \mathds{1}_{\{n > k\}}  \frac{1}{J_k} \sum_{j=1}^{J_k} \left\lvert M_j^{k} \odot (\hat{X}_{t_k, \xi_j^k} - \eta_{t_k, \xi_j^k} ) \right\rvert_2^2 + \left\lvert M_j^{k} \odot (\hat{X}_{t_k-, \xi_j^k} - \eta_{t_k-, \xi_j^k} ) \right\rvert_2^2\right] \notag \\
	&\geq \E\left[\frac{1}{n} \mathds{1}_{\{n > k\}} \left(\left\lvert M_{1}^{k} \odot (\hat{X}_{t_k, \xi_1} - \eta_{t_k, \xi_1} ) \right\rvert_2^2 + \left\lvert M_{1}^{k} \odot (\hat{X}_{t_k-, \xi_1} - \eta_{t_k-, \xi_1} ) \right\rvert_2^2\right)\right]  \notag \\
	&\geq \frac{1}{c_1^2} \left(\E\left[\mathds{1}_{\{n > k\}}\left\lvert M_{1}^k \odot\left(\hat{X}_{t_k,\xi_1} - \eta_{t_k,\xi_1}\right)\right\lvert_2\right]^2 + \E\left[\mathds{1}_{\{n > k\}}\left\lvert M_{1}^k \odot\left(\hat{X}_{t_k-,\xi_1} - \eta_{t_k-,\xi_1}\right)\right\lvert_2\right]^2\right) \notag \\
    &\geq c_2^2\frac{1}{c_1^2} \left(\E\left[\mathds{1}_{\{n > k\}}\left\lvert M_{1}^k \odot\left(\hat{X}_{t_k,\xi_1} - \eta_{t_k,\xi_1}\right)\right\lvert_1\right]^2 + \E\left[\mathds{1}_{\{n > k\}}\left\lvert M_{1}^k \odot\left(\hat{X}_{t_k-,\xi_1} - \eta_{t_k-,\xi_1}\right)\right\lvert_1\right]^2\right) \notag \\
    & \geq \frac{1}{2}\left(\frac{c_2}{c_1}\right)^2 \E\left[\mathds{1}_{\{n > k\}}\left(\left\lvert M_{1}^k \odot\left(\hat{X}_{t_k,\xi_1} - \eta_{t_k,\xi_1}\right)\right\lvert_1 + \left\lvert M_{1}^k \odot\left(\hat{X}_{t_k-,\xi_1} - \eta_{t_k-,\xi_1}\right)\right\lvert_1\right)\right]^2 \notag \\
    &\geq c_3(k)^2\frac{1}{2}\left(\frac{c_2}{c_1}\right)^2 \E\left[\mathds{1}_{\{n > k\}}\left(\left\lvert M_{1}^k \odot\left(\hat{X}_{t_k,\xi_1} - \eta_{t_k,\xi_1}\right)\right\lvert_1 + \left\lvert \hat{X}_{t_k-,\xi_1} - \eta_{t_k-,\xi_1}\right\lvert_1\right)\right]^2 \notag \\
	& = \frac{1}{2}\left( \frac{c_3(k)c_2}{c_0(k) \, c_1} \right)^2 d_k(\hat{X}, \eta)^2 > 0. \label{equ:thm1-positive expectation}
\end{align} 
In the fourth line, \Cref{assumption:7} is applied to eliminate the $J_k$ factor. Although $\eta\in\mathbb{D}$ is $\mathbb{A}$-adapted and depends on the observation collection, the random assignment of the indices $j=1,...,J_k$ means that the $j$-indexed observations are exchangeable conditional on $J_k$. Since $\eta$ and $\hat X$ are also invariant under permutations of the index, all terms of the sum in $j$ have the same conditional expectation. Consequently, the average over spatial points reduces to that of a single random observation point represented by $\xi_1^k$. Then, results follow sequentially by applying Hölder's inequality (\ref{equ:HI}), the equivalence of norms (\ref{equ:EN}), Cauchy's inequality in obtaining the algebraic identity $2a^2 +2b^2 \geq (a+b)^2$, \cref{assumption:6}, equation (\ref{eq:removing M}), the fact that $c_3(k) \leq 1$, and the definition of $d_k$ (\ref{equ: pseudo metric}). As a result, $\Psi(\eta) > \Psi(\hat{X})$, and we have shown uniqueness up to indistinguishability.

\textbf{Step 2:} Next, we show that the O-NJ-ODE \eqref{equ:O-NJ-ODE} can approximate $\hat{X}_{t,\xi}$ arbitrarily well. Our proof relies on applying $L^p$ UAT arguments in approximating the functions $f(t,\xi,O_{[0,\tau(t)]})$ and $F(\tau(t),\xi,O_{[0,\tau(t)]})$, from Assumption~\ref{assumption:4}, by networks $f_{\theta_1}$ and $\rho_{\theta_2}$. The continuous evolution network $f_{\theta_1}$ and the jump network $\rho_{\theta_2}$ should therefore have access to the full history of observations $O_{[0,\tau(t)]}$, but by design, this information is not explicitly provided as input. Instead, the previous hidden state $H_{t-}(\xi)$ must be representative of all previous observations, $O_{[0,\tau(t-)]}$, whereas the generalized kernel $\psi_{\theta_3}(t,\xi,O_{\kappa(t)})$ must be representative of current observations $O_{\kappa(t)}$. In essence, then, the hidden state can be thought of as a storage of past information, where new observations are processed by the generalized kernel and stored sequentially by the jump network. Generally speaking, it must also contain the information necessary for the readout network $g_{\theta_4}$ to extract the approximation of the conditional expectation. 

In this light, the simplest model consists of a hidden state $H_{t-}(\xi)$ that directly stores observations $O_{[0,\tau (t-)]}$ and the approximation of $\hat{X}_{t,\xi}$. We now explore the implications of this for our architecture. At observation times $t \in \{t_1,...,t_n\}$, the jump network must impute current observations $O_{\kappa(t)}$ in the hidden state and approximate $\hat{X}_{t,\xi}$. As for the generalized kernel, it is sufficient to consider its output to be equivalent to the input of the current observations, $O_{\kappa(t)}$. On the other hand, at times $t\in[0,T] \setminus \{t_1,...,t_n\}$ between observations the evolution of $\hat{X}_{t,\xi}$ is captured by the evolution network. With these notions in mind, let $H_t=(\bar H_t, \tilde H_t )$ be a division of the hidden state into two components for the purpose of estimation and storage, respectively. The neural networks are then divided similarly as $f_{\theta_1}=(\bar f_{\theta_1},\tilde f_{\theta_1})$ and $\rho_{\theta_2} = (\bar \rho_{\theta_2},\tilde \rho_{\theta_2})$. While $\bar f$ and $\bar \rho$ learn to estimate $\hat{X}_{t,\xi}$, the components $\tilde f$ and $\tilde \rho$ are concerned with the storage of observations over time. Since information is seen and stored at observation times, the entire $f$ is irrelevant to the task. Thus, $\tilde f_{
\theta_1}=0$, whereas $\tilde \rho_{\theta_2}$ alone handles the storage of observations. Defined to shift down the hidden state to impute new observations at the start of the vector, both $\tilde \rho,\tilde f\in \mathcal{N}$. The readout network is simply a projection onto $\bar H$, so it is also the case that $g_{\theta_4}\in\mathcal{N}$. In contrast, the complementary neural networks $\bar \rho$ and $\bar f$ can only approximate their target functions $F$ and $f$. 

However, to make use of $L^p$ UAT arguments there remains the issue of unbounded observation inputs. To overcome this, the unbounded information collections, $O_{\kappa(t)}$, must be restricted to a finite dimension before being passed to the network. With the notion that retaining more information leads to greater accuracy, this dimension will be inferred from a targeted approximation accuracy $\varepsilon>0$. In this regard, the approach of the proof will be general, truncating observation sets prior to being fed to the O-NJ-ODE through a standard generalized kernel $\tilde \psi_{\theta_3}$, seen as a form of preprocessing. With this approach, $\tilde \psi_{\theta_3}(t,\xi,O_{\kappa(t)}^\varepsilon)\in \mathcal{N}$ is simply a projection on the truncated input, $O_{\kappa(t)}^\varepsilon$. An alternative approach would be to use the locally aggregating kernel (\ref{eq:gen kernel}). At each spatial coordinate, this processes observations of finite dimension $d$, and is therefore capable of implicitly truncating observation sets\footnote{To replicate the implicit truncation, $\varphi^1$ can be defined to output a vector of repeated blocks of the local $d$-dimensional observations, i.e. a vector of length $d\lfloor 1/\varepsilon\rfloor$. Meanwhile, $\varphi^2$ is defined to project the output of $\varphi^1$ into the appropriate $j$-th block. It can do this by multiplying the $i$-th element of the output by the indicator function $\mathds{1}(\lfloor(i-1)/d\rfloor=j-1)$, represented by a linear combination of ReLU activations} and outputting $O_{\kappa(t)}^\varepsilon$. Despite this allowing a local aggregation without prior information loss, the proof proceeds with the former approach for generality. With this model specification, all networks except $\bar \rho$ and $\bar f$ are readily within the class $\mathcal{N}$ and parameterized by neural networks with weights $\theta_1,\theta_2,\theta_3,\theta_4\in\Theta$. Now, with bounded inputs, the existence of arbitrary approximations to $F$ and $f$ follows from the $L^p$ UAT theorem.

Beginning with the jump network, let $\nu_F$ be the push-forward measure of $\nu_F^*=du\times d\mu_\Xi\times d\P$ through the measurable map of the network inputs,

\begin{align}\label{eq:JNinputs}
    [0,1]\times \Xi \times \Omega &\rightarrow \R^{d_H+1+d_\Xi+d_s} \nonumber \\
    (U,\xi, \omega) &\mapsto \left(H_{\bar t-}(\xi, \omega),\bar t(U,\omega), \xi, \psi_{\theta_3}(\bar t(U,\omega), \xi, O_{\kappa(\bar t (U, \omega))}(\omega))\right)
\end{align}

where $\bar t :=t(U,\omega)$ as in \Cref{lemma: expectation weighted sum over t_i terms}. As will be clarified later, this reparameterization necessarily allows the measure $\nu_F$ to represent the sum over $t_i$ and $t_{i-1}$ of $F$ in \cref{assumption:4}. Additionally, $d_H=d_{\bar H}+d_{\tilde H}$ is the dimension of the hidden state and $d_s$ is the output dimension of the generalized kernel. Strictly speaking, only the storage $\tilde H_{\bar t-}(\xi)$ is required as input. However, this can be extended to the entire hidden state $H_{\bar t-}(\xi)$ as the network weights connected to $\bar H_{\bar t-}(\xi)$ can simply be set to zero. Including the previous conditional expectation estimate, $\bar H_{\bar t-}(\xi)$, will also act as a form of memory. Despite being theoretically redundant, this is likely to smooth errors and help with training. 

An issue arises when considering the dimensionality of the data. In general, both $d_H$ and $d_s$ must grow to infinity. This necessity stems from \cref{assumption:6}, which states that the number of observations $n$ and the number of observation points $J_i$ for each $i=0,...,n-1$ are integrable. As a consequence, they are almost surely finite but unbounded. Therefore, the dimensionality of each information set $O_i$ remains unknown until realized. Since the hidden state $H_{t-}(\xi)$ encodes all of $O_{[0,t-]}$ and the generalized kernel $\psi_{\theta_3}(t,\xi,O_{\kappa(t)})$ outputs the observations in $O_{\kappa(t)}$, neither $d_H$ nor $d_s$ can be fixed to a finite value.

In contrast, applying the $L^2(\nu)$ UAT theorem \citep[Theorem~1]{hornik1991approximation} requires a finite measure $\nu$ on a fixed, finite, and real-valued input space $\R^k$. Here, $\nu_F(\R^{d_H+1+d_\Xi+d_s})=1$ is indeed a finite measure, but the approach to the proof must be adjusted to accommodate an infinite-dimensional input. Ultimately, our objective is to show that increasing the O-NJ-ODE complexity enables the approximation of $\hat{X}_{t,\xi}$ with arbitrary accuracy. This motivates an approach in which a sequence of networks, each with progressively greater complexity, is constructed to achieve increasingly accurate $\varepsilon$-approximations. To enforce that inputs are of finite dimension, each network will be defined to approximate the target on a dimension-bounded but growing subspace of the input domain. 

To this end, the input \eqref{eq:JNinputs} is extended to include the current number of additional observations $n_t(\omega)$ and the current maximal discretization size $J_t^*(\omega)=\sup_{0\le i \le n_t(\omega)}J_i(\omega)$, such that

\begin{align}\label{eq:JNinputs extended}
    [0,1]\times \Xi \times \Omega &\rightarrow \R^{d_H+1+d_\Xi+d_s+1+1} \nonumber \\
    (U,\xi, \omega) &\mapsto \left(H_{\bar t-}(\xi, \omega),\bar t(U,\omega), \xi, \psi_{\theta_3}(\bar t(U,\omega), \xi, O_{\kappa(\bar t (U, \omega))}(\omega)),n_{\bar t}(\omega),J^*_{\bar t}(\omega)\right).
\end{align}

Following the suggested approach, it is now possible to focus our attention on a subset of the input space where $n$ and $J^*$ are bounded. For a fixed $\varepsilon>0$, define the $\varepsilon$-bounded subspace $A_\varepsilon=\R^{d_H}\times[0,T]\times\R^{d_{\Xi}}\times\R^{d_s}\times[0,1/\varepsilon]\times[0,1/\varepsilon]$, such that both $n\le\lfloor1/\varepsilon\rfloor$ and $J^*\le \lfloor1/\varepsilon\rfloor$, with a bound $1/\varepsilon$ that appropriately increases with precision. As a consequence, the maximal values of $n$ and $J^*$ on $A_\varepsilon$ are $\lfloor 1/\varepsilon \rfloor$. This immediately implies that any collection of information $O_i$, for $i=0,...,n-1$, can be embedded in a vector of dimension $d_s(\varepsilon)=d\lfloor1/\varepsilon\rfloor$, and that by extension, all of them can be embedded in a vector of dimension $d_H(\varepsilon)=d\lfloor1/\varepsilon\rfloor^2$. These are the precise necessities of the generalized kernel and hidden state. For a given $\varepsilon>0$, it is therefore known a priori that $d_s(\varepsilon)$ and $d_H(\varepsilon)$ are valid and importantly finite dimensions of the generalized kernel and hidden state. On $A_\varepsilon$, no information of the observations will be lost with this choice of dimension. In this case, a valid choice of input space for $A_\varepsilon$ would be $\R^k$, where the original infinite-dimensional inputs are truncated to dimension $k=d\lfloor1/\varepsilon\rfloor^2+1+d_\Xi+d\lfloor1/\varepsilon\rfloor + 1 + 1$, which is both fixed and finite. This support encapsulates all of $A_\varepsilon$ because any potential set of realizations is an embedding of this space. In particular, it is compatible with the $L^p$ UAT theorem, which enables $L^2$-approximation on $A_\varepsilon$. These observations indicate that the inadmissible approximation of the  target $F$, on the entire domain $A=A_\varepsilon \cup A_\varepsilon^c$, should be reformulated in terms of one on the $\varepsilon$-bounded subspace $A_\varepsilon$. 

Focusing on an arbitrary observation collection $O_{[0,t]}$ up to observation time $t$ and continuing with a fixed $\epsilon>0$, let $F(t,\xi,O_{[0,t]})=F(O_{[0,t]})\cdot\mathds{1}_{A_\varepsilon}+F(O_{[0,t]})\cdot\mathds{1}_{A_\varepsilon^c}$. Then, the restriction of $F$ on $A_\varepsilon$ is $\hat{F}:=F|_{A_\varepsilon}:A_\varepsilon\rightarrow\R^{d_X}$, now with a truncated input $O_{[0,t]}^\varepsilon \in \R^{d_H(\varepsilon)}$ representing the fixed and finite dimensional collection of all truncated observations $O_i^\varepsilon$ for $i=0,1,...,\kappa(t)$ up to time $t$. Evaluation of the O-NJ-ODE splits this collection into $(O_{[0,\tau(t-)]}^\varepsilon,O_{\kappa(t)}^\varepsilon))$ when passed to the jump network. The first component of past observations is captured in the hidden state $H_{t-}(\xi)$ and the current time $t$ observations are processed by the generalized kernel. These truncated observations $O_{\kappa(t)}^\varepsilon\in \R^{d_s(\varepsilon)}$ represent a padded projection into the fixed and finite dimension $d_s(\varepsilon)$. More specifically, if the original collection of information $O_{\kappa(t)}$ contains fewer than $d_s(\varepsilon)$ observations, $O_{\kappa(t)}^\varepsilon$ will consist of all those observations in addition to the following trailing zeros, up to dimension $d_s(\varepsilon)$. On the other hand, if $O_{\kappa(t)}$ contains more observations than $d_s(\varepsilon)$, $O_{\kappa(t)}^\varepsilon$ will consist only of observations up to dimension $d_s(\varepsilon)$. Thus, we achieve the desired $L^p$ UAT compatibility on $A_\varepsilon$, and \citet[Theorem~1]{hornik1991approximation} implies that there exists $m_2(\varepsilon)=m_2\in\N$ and network weights $\theta_2^*\in \Theta_{m_2}^2$ such that 
\begin{equation*}
    \left\|\bar\rho_{\theta_2^*} - \hat{F}(O^\varepsilon_{[0, \tau(\bar t)]})\mathds{1}_{\{n_{\bar t}\leq 1/\varepsilon, J_{\bar t}^*\leq 1/\varepsilon\}} \right\|_{L^2(A_\varepsilon,\,\R^{d_X},\;\nu_F)} \leq \varepsilon.
\end{equation*}
Letting $D_F=[0,1]\times\Xi\times\Omega$, the $L^2$-distance between $F$ and the jump network $\bar \rho_{\theta^*_2}$ can be reformulated as
\begin{align} \label{eq:jump UAT}
    &\left\|\bar\rho_{\theta_2^*}(H_{\bar t-}(\xi),\bar t, \xi, O_{\kappa(\bar t)}^\varepsilon) - F(\bar t, \xi,O_{[0, \tau(\bar t)]})\right\|_{L^2(D_F,\,\R^{d_X},\;d\nu^*_F)} \nonumber \\
    &= \left\|\bar\rho_{\theta_2^*} - \left(F(O_{[0, \tau(\bar t)]})\mathds{1}_{\{n_{\bar t}\leq 1/\varepsilon, J_{\bar t}^*\leq 1/\varepsilon\}} + F(O_{[0, \tau(\bar t)]})(1-\mathds{1}_{\{n_{\bar t}\leq 1/\varepsilon, J_{\bar t}^*\leq 1/\varepsilon\}})\right) \right\|_{L^2(D_F,\,\R^{d_X},\;d\nu_F^*)} \nonumber \\
    & \le \left\|\bar\rho_{\theta_2^*} - F(O_{[0, \tau(\bar t)]})\mathds{1}_{\{n_{\bar t}\leq 1/\varepsilon, J_{\bar t}^*\leq 1/\varepsilon\}} \right\|_{L^2(D_F,\,\nu_F^*)} + \left\| F(O_{[0, \tau(\bar t)]})(1-\mathds{1}_{\{n_{\bar t}\leq 1/\varepsilon, J_{\bar t}^*\leq 1/\varepsilon\}}) \right\|_{L^2(D_F,\,\nu_F^*)}\nonumber \\
    & \le \left\|\bar\rho_{\theta_2^*} - \hat{F}(O^\varepsilon_{[0, \tau(\bar t)]})\mathds{1}_{\{n_{\bar t}\leq 1/\varepsilon, J_{\bar t}^*\leq 1/\varepsilon\}} \right\|_{L^2(A_\varepsilon,\,\nu_F)} + \left\| F(O_{{[0, \tau(\bar t)]}})(1-\mathds{1}_{\{n_{\bar t}\leq 1/\varepsilon, J_{\bar t}^*\leq 1/\varepsilon\}}) \right\|_{L^2(D_F,\,\nu_F^*)}\nonumber \\
    & \le \varepsilon + \left\| F(O_{[0, \tau(\bar t)]})(1-\mathds{1}_{\{n_{\bar t}\leq 1/\varepsilon, J_{\bar t}^*\leq 1/\varepsilon\}}) \right\|_{L^2(D_F,\,\nu_F^*)}.
\end{align}
In the fourth inequality we push through the measurable input map (\ref{eq:JNinputs extended}). Its image is contained in the fixed and finite-dimensional Euclidean input space $A_\varepsilon$, on which we can apply the $L^P$ UAT theorem with respect to the pushforward measure. Arbitrary accuracies $\varepsilon>0$ can be obtained by taking the corresponding $A_\varepsilon$ as the approximation domain of the network. However, for the desired approximation to hold with arbitrary accuracy, it remains to show that the second term approaches zero as $\varepsilon$ decreases. Given that $F$ is integrable by the finite expectation assumption (\ref{equ:assumption4 bound}) and that the complement of the $\varepsilon$-bounded input space $A_\varepsilon^c\rightarrow \emptyset$ as $\varepsilon\rightarrow0$, this is indeed the case. 
More precisely, we have 
\begin{equation*}
    \mathds{1}_{A_\varepsilon^c} \leq \mathds{1}_{\{n \geq 1/\varepsilon\}} + \mathds{1}_{\{J^* \geq 1/\varepsilon\}} \xrightarrow[\varepsilon \to 0]{\P-a.s.} 0,
\end{equation*}
since $n, J^*$ are integrable by \Cref{assumption:6}.

Moving on to the continuous evolution network, we follow the same approach, letting $\nu_f$ be the finite push-forward measure of $\nu_f^* =dt\times d\mu_{\xi}\times d\P$ defined by the measurable map to the network inputs,

\begin{align}
    [0,T]\times \Xi \times \Omega &\rightarrow \R^{d_H+1+1+d_\Xi+1+1} \nonumber \\
    (t,\xi, \omega) &\mapsto \left(H_{t-}(\xi, \omega), t, \tau(\omega,t), \xi, n_t(\omega), J_t^*(\omega)\right).
\end{align}

As before, the input has been extended to include $n_t(\omega)$ and $J^*_t(\omega)$, and the measure $\nu_f$ is finite with $\nu_f(\R^{d_H+1+1+d_\Xi+1+1})=T$. For compatibility with the $L^p$ UAT theorem, the $\varepsilon$-bounded input space is defined as $B_{\epsilon}=\R^{d_H}\times[0,T]^2\times\Xi\times[0,1/\varepsilon]\times[0, 1/\varepsilon]$, and again we partition the support as $B=B_\varepsilon\cup B_\varepsilon^c$ to reformulate the $L^2$-approximation of $f$ in terms of the UAT-compatible space $B_\varepsilon$. 
Here, the restriction of $f$ on $B_\varepsilon$ is defined as $\hat{f}:=f|_{B_\varepsilon}:B_\varepsilon\rightarrow \R^{d_X}$. This function takes the truncated input $O_{[0,\tau(t)]}^\varepsilon\in\R^{d_H(\varepsilon)}$, which is the projection of $O_{[0,\tau(t)]}$ on its first $d_H(\varepsilon)$ elements. For the network, the informational equivalent is $\tilde H_{t-}^\varepsilon\in\R^{d_H(\varepsilon)}$, which denotes the information accumulated by the jump network over time. For a particular $\omega\in\Omega$, $\tilde H_{t-}^\varepsilon$ is equal to $O_{[0,\tau(t)]}^\varepsilon\in\R^{d_H(\varepsilon)}$. 
Then, for any $\epsilon >0$, there exists an $m_1(\varepsilon)=m_1 \in\N$ with network weights $\theta_1^* \in \Theta_{m_1}^1$ such that 
\begin{equation*}
    \left\|\bar f_{\theta_1^*} - \hat{f}(O_{[0, \tau(t)]}^\varepsilon)\mathds{1}_{B_\varepsilon}\right\|_{L^2(B_\varepsilon,\,\R^{d_X},\;\nu_f)} \leq \varepsilon.
\end{equation*}
Hence, with $D_f=[0,T]\times\Xi \times \Omega$ we can bound the approximation of $f$ as
\begin{align} \label{eq:evol UAT}
    &\left\|\bar f_{\theta_1^*}\left(H_{t-}^\varepsilon(\xi), t, \tau(t), \xi\right) - f\left(t, \xi, O_{[0, \tau(t)]}\right) \right\|_{L^2(D_f,\,\R^{d_X};\;\nu^*_f)} \nonumber \\
    &= \left\|\bar f_{\theta_1^*} - \left(f(O_{[0, \tau(t)]})\mathds{1}_{\{n_{\bar t}\leq 1/\varepsilon, J_{\bar t}^*\leq 1/\varepsilon\}} + f(O_{[0, \tau(t)]})(1-\mathds{1}_{\{n_{\bar t}\leq 1/\varepsilon, J_{\bar t}^*\leq 1/\varepsilon\}})\right)\right\|_{L^2(D_f,\,\R^{d_X};\;\nu^*_f)} \nonumber \\
    &\le \left\|\bar f_{\theta_1^*} - f(O_{[0, \tau(t)]})\mathds{1}_{\{n_{\bar t}\leq 1/\varepsilon, J_{\bar t}^*\leq 1/\varepsilon\}}\right\|_{L^2(D_f,\,\nu_f^*)} + \left\|f(O_{[0, \tau(t)]})(1-\mathds{1}_{\{n_{\bar t}\leq 1/\varepsilon, J_{\bar t}^*\leq 1/\varepsilon\}})\right\|_{L^2(D_f,\,\nu_f^*)} \nonumber \\
    &\le \left\|\bar f_{\theta_1^*} - \hat{f}(O_{[0, \tau(t)]}^\varepsilon)\mathds{1}_{\{n_{\bar t}\leq 1/\varepsilon, J_{\bar t}^*\leq 1/\varepsilon\}}\right\|_{L^2(B_\varepsilon,\,\nu_f)} + \left\|f(O_{[0, \tau(t)]})(1-\mathds{1}_{\{n_{\bar t}\leq 1/\varepsilon, J_{\bar t}^*\leq 1/\varepsilon\}})\right\|_{L^2(D_f,\,\nu_f^*)} \nonumber \\
    &\le \varepsilon + \left\|f(O_{[0, \tau(t)]})(1-\mathds{1}_{\{n_{\bar t}\leq 1/\varepsilon, J_{\bar t}^*\leq 1/\varepsilon\}})\right\|_{L^2(D_f,\,\nu_f^*)}.
\end{align}
As was the case for the jump network, the integrability of $f$ and the vanishing behavior of $B_\varepsilon^c$ for decreasing $\varepsilon$ imply that $f$ can be approximated arbitrarily well. Indeed, the part of $f$ that contributes to the second term of (\ref{eq:evol UAT}) becomes negligible, since

\begin{equation*}
    \mathds{1}_{B_\varepsilon^c} \leq \mathds{1}_{\{n \geq 1/\varepsilon\}} + \mathds{1}_{\{J^* \geq 1/\varepsilon\}} \xrightarrow[\varepsilon \to 0]{\P-a.s.} 0,
\end{equation*}
by integrability of $n, J^*$ from \Cref{assumption:6}.

In the limit of finer approximations, the $\varepsilon$-bounded subspaces $A_\varepsilon$ and $B_\varepsilon$ will approach the entirety of their support, and the error on their vanishing complements will be negligible. Having established the existence of arbitrary $\varepsilon$-approximations for both $F$ and $f$, we may progress towards formalizing the O-NJ-ODE approximation of the conditional expectation, $\hat{X}_{t,\xi}$, in regard to minimizing the objective function (\ref{equ:Phi}). 

So far, we have considered approximations of $F$ and $f$ on UAT-compatible domains up to a fixed accuracy, $\varepsilon>0$, and extended these results to approximations on all their supports under specified constraints on complexity. In both cases, attaining lower $\varepsilon$-approximations requires a more complex network. As accuracy increases, the $1/\varepsilon$ bounds that define $A_\varepsilon$ and $B_\varepsilon$ approach infinity, as do the dimensions $d_H(\varepsilon)$ and $d_s(\varepsilon)$. Better approximations require more memory, both in storing old observations and in processing new ones. Reducing the approximation error also requires a decrease in network bias. Networks of greater depth or width allow for better approximations. For example, newly introduced weights can simply be set to zero to obtain previous approximations. Hence, requiring a lower $\varepsilon$ error will result in an increase in the complexities $m_1(\varepsilon)$ and $m_2(\varepsilon)$. Given these inverse relationships, we formalize the convergence in error by the scaling complexity of the network. Letting $m=\max(m_1,m_2,d_H,d_s)$ quantify the complexity of our network, the best achievable approximation error for complexity $m$ can be denoted by $\varepsilon_m$. The 2-norms of the network weights are included in this definition for later use in proving \Cref{thm:MC convergence Yt}. By the previous arguments, $m\rightarrow\infty$ allows the complexity of the O-NJ-ODE to increase and there will be a concurrent decrease in the attainable approximation error such that $\varepsilon_m\rightarrow0$. Now, it remains to show that the O-NJ-ODE \eqref{equ:O-NJ-ODE}, trained with respect to the objective function \eqref{equ:Phi}, is indeed capable of arbitrarily approximating $\hat{X}_{t,\xi}$ in this limit.

To begin, it will be shown that the loss of the approximating network $\Phi(\theta^*_m)$ converges to the global minimum $\Psi(\hat{X}_{t,\xi})$ as complexity $m\rightarrow\infty$. Here, $\theta_m^*=(\theta_1^*, \theta_2^*, \theta_3, \theta_4) \in \Theta_m$, and by \Cref{lemma: L2 optimality of conditional expectation} the objective function can be reformulated as

\begin{equation} \label{eq:deviation term}
\Phi(\theta^*_m) = \Psi(\hat{X}) +\E\left[\frac{1}{n}\sum_{i=0}^{n-1}\frac{1}{J_i}\sum_{j=1}^{J_i}\left|M_j^i\odot\left(\hat{X}_{t_i,\xi_j^i}-Y_{t_i}^{\theta^*_m}(\xi_j^i)\right)\right|_2^2 + \left|M_j^i\odot\left(\hat{X}_{t_i-,\xi_j^i}-Y_{t_i-}^{\theta_m^*}(\xi_j^i)\right)\right|_2^2\right].
\end{equation}

Minimizing the objective is therefore equivalent to showing that this second term, henceforth referred to as the deviation term, approaches zero. Thus, we bound the $L^2$-distance between the O-NJ-ODE output $Y_t^{\theta^*_m}(\xi)=Y_{t,\xi}^{\theta_m^*}$ and $\hat{X}_{t,\xi}$. For observation times $t\in\{t_1,...,t_n\}\subset [0,T]$ and spatial coordinate $\xi\in\Xi$, the distance simplifies to 

\begin{align}\label{eq:jump inequality}
    \left|Y_{t,\xi}^{\theta_m^*}-\hat{X}_{t,\xi}\right|_2 = \left|\bar\rho_{\theta_2^*}(H_{t-}(\xi), t,\xi,\psi_{\theta_3^*}(t,\xi,O_{\kappa(t)}))-F(t,\xi,O_{[0, \tau(t)]})\right|_2.
\end{align}

For the remaining times, $t\in [0,T] /\{t_1,...,t_n\}$, where both the jump and evolution networks are relevant, the distance can be reformulated as

\begin{align}\label{eq:evol inequality}
    \left|Y_{t,\xi}^{\theta_m^*}-\hat{X}_{t,\xi}\right|_2 &= \left|Y_{\tau(t)}^{\theta_m^*}-\hat{X}_{\tau(t)} +\int_{\tau(t)}^t\left(\bar f_{\theta_1^*}\left(H_{s-}(\xi), s, \tau(t), \xi\right)-f\left(s, \xi, O_{[0, \tau(t)]}\right)\right)ds \right|_2 \nonumber \\
    &\le \left|Y_{\tau(t), \xi}^{\theta_m^*}-\hat{X}_{\tau(t), \xi}\right|_2 + \int_{\tau(t)}^t\left|\bar f_{\theta_1^*}-f\right|_2 ds \nonumber\\
    &\le \left|\bar\rho_{\theta^*_2}(\tau(t),\xi)-F(\tau(t),\xi)\right|_2 + \int_{0}^T\left|\bar f_{\theta_1^*}\left(H_{t-}(\xi),t, \tau(t), \xi\right)-f\left(t,\xi,O_{[0,\tau(t)]}\right)\right|_2 dt
\end{align}

by applying the triangle inequality, standard integral results, and substituting the inequality obtained at jump times (\ref{eq:jump inequality}). For brevity, only time and space inputs are explicitly written for the jump part. Note that this final inequality also holds for any $t\in[0,T]$.

A structural similarity is already apparent between inequalities (\ref{eq:jump inequality}, \ref{eq:evol inequality}) and the $L^2$-norms with respect to the pushforward measures $\nu_F$ and $\nu_f$. Neglecting the mask $M_j^i$ in the deviation term of (\ref{eq:deviation term}) and substituting for the inequalities yields an expression in terms of these norms that is amenable to convergence analysis. In the following, the networks $\bar \rho_{\theta_2^*}$ and $\bar f_{\theta_1^*}$ are taken to approximate their targets with error $\varepsilon_m>0$, so that

\begin{align} \label{eq:DCT loss convergence}
    &\E\left[\frac{1}{n}\sum_{i=0}^{n-1}\frac{1}{J_i}\sum_{j=1}^{J_i}\left|M_j^i\odot\left(\hat{X}_{t_i,\xi_j^i}-Y_{t_i}^{\theta^*_m}(\xi_j^i)\right)\right|_2^2\right] + \E\left[\frac{1}{n}\sum_{i=0}^{n-1}\frac{1}{J_i}\sum_{j=1}^{J_i}\left|M_j^i\odot\left(\hat{X}_{t_i-,\xi_j^i}-Y_{t_i-}^{\theta_m^*}(\xi_j^i)\right)\right|_2^2\right] \nonumber \\
    &\quad\;\;\le \E\left[\frac{1}{n}\sum_{i=0}^{n-1}\frac{1}{J_i}\sum_{j=1}^{J_i}\left|Y_{t_i}^{\theta^*_m}(\xi_j^i)-\hat{X}_{t_i,\xi_j^i}\right|_2^2\right] + \E\left[\frac{1}{n}\sum_{i=0}^{n-1}\frac{1}{J_i}\sum_{j=1}^{J_i}\left|Y_{t_i-}^{\theta_m^*}(\xi_j^i)-\hat{X}_{t_i-,\xi_j^i}\right|_2^2\right]  \nonumber\\
    &\quad\;\;\le \E\left[\frac{1}{n}\sum_{i=0}^{n-1}\frac{1}{J_i}\sum_{j=1}^{J_i}\left|[\bar\rho_{\theta_2^*}-F](t_i,\xi_j^i) \right|_2^2 + \left(\left|[\bar\rho_{\theta_2^*}-F](t_{i-1},\xi_j^i)\right|_2 + \int_0^T\left|[\bar f_{\theta_1^*}-f](t,\xi_j^i)\right|_2dt\right)^2\right]  \nonumber\\
    &\quad\;\; \le 2\left(\E\left[\frac{1}{n}\sum_{i=0}^{n-1}\frac{1}{J_i}\sum_{j=1}^{J_i}\left|[\bar\rho_{\theta_2^*}-F](t_i,\xi_j^i) \right|_2^2 +\left|[\bar\rho_{\theta_2^*}-F](t_{i-1},\xi_j^i) \right|_2^2 + \left( \int_0^T\left|[\bar f_{\theta_1^*}-f](t,\xi_j^i)\right|_2dt\right)^2\right]\right)  \nonumber\\
    &\quad\;\;\le 2\left(\E_{\nu_F}\left[\left|[\bar\rho_{\theta_2^*}-F](\bar t,\xi)\right|_2^2\right] + T\E_{\nu_f}\left[\left|[\bar f_{\theta_1^*}-f](\xi)\right|_2^2\right]\right)  \nonumber\\
    &\quad\;\;= 2\left(\left\|[\bar\rho_{\theta_2^*}-F](\bar t,\xi)\right\|^2_{L^2(A,\nu_F)}+T\left\|[\bar f_{\theta_1^*}-f](\xi)\right\|^2_{L^2(B,\nu_f)}\right)  \nonumber\\
    &\quad\;\;\le 2\left( O(\varepsilon_m^2) + \left\| F(O_{[0,\tau(\bar t)]})\mathds{1}_{A_{\varepsilon_m}^c} \right\|^2_{L^2(A,\nu_F)} + T\left\|f(O_{[0, \tau(t)]})\mathds{1}_{B_{\varepsilon_m}^c}\right\|^2_{L^2(B,\nu_f)}\right).
\end{align}

Inequalities (\ref{eq:jump inequality}, \ref{eq:evol inequality}) are substituted in the second inequality, and applying the Cauchy-Schwarz inequality yields the third. The fourth inequality utilizes the definitions of the push-forward measures $\nu_F$ and $\nu_f$, \cref{assumption:4} for integrability of terms in $F$ and $f$, and \cref{assumption:7} which states that all $\xi_j^i$ are i.i.d. copies of $\xi\sim\mu_{\Xi}$. After reformulating the expression in terms of the $L^2$-norms, the UAT approximations (\ref{eq:jump UAT}, \ref{eq:evol UAT}) are substituted for the final result. With \cref{assumption:4} and the vanishing behavior of $A_{\varepsilon_m}^c$ and $B_{\varepsilon_m}^c$, as $m\rightarrow\infty$ the terms in the expectation (\ref{eq:DCT loss convergence}) go to zero. Applying the dominated convergence theorem, the limit can be passed outside of the expectation and the deviation term will also approach zero by extension. This establishes the convergence in loss of the O-NJ-ODE output to the conditional expectation.

Next, consider the minimal O-NJ-ODE loss $\Phi(\theta_m^{\text{min}})$ with $\theta_m^{min}\in\argmin_{\theta_m\in\Theta_m}\Phi(\theta_m)$, which is guaranteed to exist as $\Phi(\theta)$ is continuous on the compact set $\Theta_m$. Combining the results (\ref{eq:deviation term}) and (\ref{eq:DCT loss convergence}), the minimal loss can be squeezed as

\begin{align*}
    \min_{\eta\in \mathbb{D}} & \Psi(\eta)\le \Phi(\theta_m^{\text{min}}) \le \Phi(\theta_m^*) \\
    &= \Psi(\hat{X}) +\E\left[\frac{1}{n}\sum_{i=0}^{n-1}\frac{1}{J_i}\sum_{j=1}^{J_i}\left|M_j^i\odot\left(\hat{X}_{t_i,\xi_j^i}-Y_{t_i}^{\theta^*_m}(\xi_j^i)\right)\right|_2^2 + \left|M_j^i\odot\left(\hat{X}_{t_i-,\xi_j^i}-Y_{t_i-}^{\theta_m^*}(\xi_j^i)\right)\right|_2^2\right]  \\
    &\le \Psi(\hat{X}) +\E\left[\frac{1}{n}\sum_{i=0}^{n-1}\frac{1}{J_i}\sum_{j=1}^{J_i}\left|\hat{X}_{t_i,\xi_j^i}-Y_{t_i}^{\theta^*_m}(\xi_j^i)\right|_2^2 + \left|\hat{X}_{t_i-,\xi_j^i}-Y_{t_i-}^{\theta_m^*}(\xi_j^i)\right|_2^2\right] \\
    &\le \Psi(\hat{X}) + 2\left( \varepsilon_m + \left\| F(O_{[0, \tau(\bar t)]})\mathds{1}_{A_{\varepsilon_m}^c} \right\|^2_{L^2(A,\nu_F)} + \left\|f(O_{[0, \tau(t)]})\mathds{1}_{B_{\varepsilon_m}^c}\right\|^2_{L^2(B,\nu_f)}\right)\\
    &\xrightarrow[m\rightarrow\infty]{} \Psi(\hat{X}). \\
\end{align*}

Finally, as $\Psi(\hat{X}) =\min_{\eta\in\mathbb{D}}\Psi(\eta)$, the previous arguments culminate in the result

\begin{equation}
    \min_{\eta\in \mathbb{D}}\Psi(\eta)\le \Phi(\theta_m^{\text{min}}) \le \Phi(\theta_m^*) \xrightarrow[m\rightarrow\infty]{} \min_{\eta\in\mathbb{D}}\Psi(\eta).
\end{equation}

This confirms that by minimizing the objective function, the O-NJ-ODE converges (in loss) to the conditional expectation up to indistinguishability.

\textbf{Step 3:} It remains to prove precisely how the O-NJ-ODE converges to the conditional expectation. To confirm that the convergence is, in fact, with respect to the metrics $d_k$, it will be shown that $\lim_{m\rightarrow\infty}d_k(Y_t^{\theta_m^{min}},\hat{X})=0$. First, we note that the term

\begin{equation}\label{eq:loss diff}
    \Phi(\theta_m^{min})-\Psi(\hat{X}) = \E\left[\frac{1}{n}\sum_{i=0}^{n-1}\frac{1}{J_i}\sum_{j=1}^{J_i}\left|M_j^i\odot\left(\hat{X}_{t_i,\xi_j^i}-Y_{t_i}^{\theta^\text{min}_m}(\xi_j^i)\right)\right|_2^2 + \left|M_j^i\odot\left(\hat{X}_{t_i-,\xi_j^i}-Y_{t_i-}^{\theta_m^\text{min}}(\xi_j^i)\right)\right|_2^2\right]
\end{equation}

from (\ref{eq:deviation term}) approaches zero with increasing complexity. Thus, applying \cref{equ:HI,equ:EN,eq:removing M,eq:loss diff} yields 

\begin{equation} \label{eq: convergence in metric}
    d_k(Y^{\theta_m^{\text{min}}}, \hat{X}) \le \frac{\sqrt{2} \, c_0(k) \, c_1}{ c_2 \, c_3(k)}\left(\Phi(\theta_m^{min})-\Psi(\hat{X})\right)^{1/2} \xrightarrow[]{m\rightarrow\infty} 0,
\end{equation}
which completes the proof.
\end{proof}

\subsection{Convergence of the Monte Carlo Approximation}\label{sec:Convergence of the Monte Carlo approximation}
Assuming now that the complexity $m\in\N$ of the neural network is fixed, we study the convergence of the Monte Carlo approximation (\ref{equ:appr loss function}) when the number of sample paths $N\in\N$ increases. The theorem follows from \citet{heiss2024nonparametric} with changes in notation, establishing almost sure uniform convergence of the MC approximation to the objective function on the compact set $\Theta_m$. It is also shown that there is almost sure convergence in both its minimum value and minimizer. 
\\
\begin{theorem}
\label{thm:MC convergence Yt}
Let $\theta^{\min}_{m,N} \in \Theta^{\min}_{m,N} := \argmin_{\theta \in \Theta_m}\{ \hat\Phi_N(\theta)\}$ for every $m, N \in \N$,  and assume either that both neural networks $f_{\theta_1}$ and $\rho_{\theta_2}$ or the readout network $g_{\theta_4}$ have a bounded activation function. 
Then, for every $m \in \N$, $\P\text{-a.s.}$ 
\begin{equation*}
\hat\Phi_N \xrightarrow{N \to \infty} \Phi \quad \text{uniformly on } \Theta_m.
\end{equation*}
Moreover, for every $m \in \N$, we have a.s.,
\begin{equation*}
\hat\Phi_N(\theta^{\min}_{m,N}) \xrightarrow{N \to \infty} \Phi(\theta^{\min}_{m}) \quad \text{and} \quad \Phi(\theta^{\min}_{m,N}) \xrightarrow{N \to \infty} \Phi(\theta^{\min}_{m}) .
\end{equation*}
In particular, one can define an increasing random sequence $(N_m)_{m \in \N}$ in $\N$ such that for every $0 \leq k \leq K$ we have almost surely that $Y^{\theta_{m, N_m}^{\min}}$ converges to $\hat{X}$ 
in the metric $d_k$  as $m \to \infty$. \\
\end{theorem}

As in \citet{heiss2024nonparametric}, we define the separable Banach space $\mathcal{S} := \ell^1(\R^d) = \{x=(x_i)_{i\in\N} \; \vert \; x_i\in \R^d,\; \|x \|_{\mathcal{S}} < \infty \}$ with the norm $\|x\|_{\mathcal{S}} := \sum_{i \in \N} \| x_i \|_{\ell^1}=\sum_{i,j \in \N}\lvert x_{i,j} \rvert_{2}$ and suitable dimension $d$ of information realized at observation times, the same as before. To construct the MC approximation and define the ex post information setting, the  instantaneous loss and realized observations are defined as

\begin{equation*}
\phi(x,y,z,m) :=  \left\lvert m \odot ( x - y ) \right\rvert_2^2 + \left\lvert m \odot (x - z) \right\rvert_2^2 
\end{equation*}
and $\zeta_l := (\zeta_{0}^{(l)}, \dotsc, \zeta_{n^{(l)}-1}^{(l)}, 0, \dotsc)$, where $\zeta_{k}^{(l)}:=(\zeta_{k,1}^{(l)}, \dotsc,\zeta_{k,J_k^{(l)}}^{(l)}, 0, \dotsc)$ with 
$$\zeta_{k, j}^{(l)} := ((M_j^k)^{(l)}\odot (X_{t_k^{(l)}, (\xi_j^k)^{(l)}})^{(l)}, t_k^{(l)},(\xi^k_j)^{(l)}, J_k^{(l)},(M^k_j)^{(l)}) \in \R^d.$$ 
The random variables $t_k^{(l)}$, $(M_{j}^k)^{(l)}$, $(\xi_j^k)^{(l)}$, $J_k^{(l)}$, and $(M_j^k)^{(l)}\odot (X_{t_k^{(l)}, (\xi_j^k)^{(l)}})^{(l)}$ describe the $l$-th sample of the realized data, as defined in Section \ref{sec:Objective Function}. Henceforth, we let $(X_{t_k^{(l)}, (\xi_j^i)^{(l)}})^{(l)}$ denote the observations with $0$ entries at unobserved coordinates.
Taking $n(\zeta_l) := |\{ t_k^{(l)} | t_k^{(l)} < \infty \}|$, $t(\zeta_{i,j}^{(l)}) = t_i^{(l)}$ and  $J_k(\zeta_l):= |\{ (k,j) | t(\zeta_{k,j}^{(l)}) = t_k^{(l)} \}|$ for $k=0,\dotsc,n(\zeta_l)-1$, we have $\P$-almost surely that $n^{(l)} = n(\zeta_l)$ and $J_k^{(l)}=J_k(\zeta_l)$. Furthermore, we have that $t_k(\zeta_l):= t_{k}^{(l)}$, $X_{t_k,\xi_j^k} (\zeta_l) := X_{t_k^{(l)}, (\xi^k_j)^{(l)}}^{(l)}$, $\xi_j^k(\zeta_l)=(\xi_j^k)^{(l)}$, and $M^k_j(\zeta_l) := (M^k_j)^{(l)}$, where $\zeta_l$ are i.i.d. random variables taking values in $\mathcal{S}$. Writing $Y_{t_k,\xi_j^k}^{\theta}(\zeta_l):=(Y_{t_k^{(l)},(\xi_j^k)^{(l)}}^{\theta})^{(l)}$ to make the dependence of $Y$ on the input and the weight $\theta$ explicit, the sample-wise loss can be defined as

\begin{equation*}
h(\theta, \zeta_l) := \frac{1}{n(\zeta_l)}\sum_{i=0}^{n(\zeta_l)-1}\frac{1}{J_i(\zeta_l)}\sum_{j=1}^{J_i(\zeta_l)}  \phi \left( X_{t_i,\xi_j^i}(\zeta_l), Y^\theta_{t_i,\xi_j^i}(\zeta_l), Y^\theta_{t_i-,\xi_j^i}(\zeta_l), M_j^i(\zeta_l) \right).
\end{equation*}

Before proceeding with the proof, we introduce two core lemmas from \citet{heiss2024nonparametric} for the arguments to follow. 
\begin{lemma}\label{lem:properties for MC conv thm}
Almost surely the random function $\theta\in\Theta_m \mapsto Y_{t,\xi}^{\theta}$ is uniformly continuous for every $t \in [0,T]$.
\end{lemma}

\begin{proof}
Continuity of $\theta\in \Theta_m \mapsto Y_{t,\xi}^{\theta}$ follows from the fact that the activation functions are continuous. Since ${\Theta}_m$ is compact, this extends to uniform continuity using standard arguments.
\end{proof}

\begin{lemma}
\label{lemma:convergencelocallyuniform}
Let $(\zeta_i)_{i \geq 1}$ be a sequence of i.i.d. random variables with values in $\mathcal{S}$ and $h:\mathbb{R}^D\times \mathcal{S}\to \mathbb{R}$ be a measurable function.
Assume that almost surely, the function $\theta\in \mathbb{R}^D \mapsto h(\theta, \zeta_1)$ is continuous and for all $C>0$, $\E(\sup_{|\theta|_2 \leq C}|h(\theta, \zeta_1)|)< + \infty$. Then, $\P\text{-a.s.}$ $f_N:  \mathbb{R}^D \to \R, \theta  \mapsto \frac{1}{N}\sum_{i=1}^{N}h(\theta, \zeta_i)$ converges locally uniformly to the continuous function $f: \mathbb{R}^D \to \R, \theta \mapsto \E(h(\theta, \zeta_1))$, i.e.,
\begin{equation*}
\lim_{N\to\infty} \sup_{|\theta|_2\leq C} \left|\frac{1}{N}\sum_{i=1}^{N}h(\theta, \zeta_i) -   \E(h(\theta, \zeta_1))\right| = 0 \qquad \P\text{-}a.s.
\end{equation*}
Moreover, let $K \subset \R^D$ be compact and define the random variables $v_n := \inf_{x\in K} f_n(x)$. We consider a minimizing sequence of random variables $(x_n)_{n=0}^{\infty}$, given by $f_n(x_n) = \inf_{x\in K} f_n(x) $ and let $v^\star = \inf_{x\in K}f(x)$ and $\mathcal{K}^\star = \{ x\in K: f(x) =v^\star \}$. Then $v_n \to v^\star$ and $d(x_n, \mathcal{K}^\star) \to 0 $ almost surely, where $d(\cdot,\cdot)$ is the standard Euclidean metric on $R^D$.
\end{lemma}
\begin{proof}
This is a consequence of \cite[Corollary 7.10]{ledoux1991m} and \cite[Sec. 2.6, Lemma A1 \& the discussion of Theorem A1]{rubinstein1993discrete}.
\end{proof}

\begin{proof}[Proof of Theorem \ref{thm:MC convergence Yt}.]
First, we note that the O-NJ-ODE output $Y^\theta_{t,\xi}$ with $\theta\in\Theta_m$ is defined by the integration over neural networks with at least one bounded activation function, as assumed in the theorem. Consequently, $Y^\theta_{t,\xi}$ itself is bounded in terms of the bound of these activations, the time period $T$, and some constant depending on the complexity $m$ and the architecture of the neural networks. Moreover, the activations are assumed to be Lipschitz continuous and each linear layer corresponds to a map with operator norm bounded by a constant depending on $m$, followed by a Lipschitz activation. Composing such layers leads to at most linear growth with respect to the bound of the bounded activation. In particular, we have $|Y^\theta_{t,\xi}(\zeta_l)|_2 \leq \tilde B(m,T)$ for all $t \in [0,T]$, $\xi \in \Xi$, and $\theta \in \Theta_m$ for some constant $\tilde B(m,T)$, which also depends on the maximum bound among the bounded activation functions. Hence,
\begin{align*}
\phi &\left( X_{t_i,\xi_j^k}(\zeta_l), Y^\theta_{t_i,\xi_j^k}(\zeta_l), Y^\theta_{t_i-,\xi_j^k}(\zeta_l), M_j^i(\zeta_l) \right) \\
	&= \left\lvert M_j^i(\zeta_l) \odot (X_{t_i,\xi_j^i}(\zeta_l) - Y^\theta_{t_i,\xi_j^i}(\zeta_l) ) \right\rvert_2^2 + \left\lvert M_j^i(\zeta_l) \odot ( X_{t_i,\xi_j^i}(\zeta_l) - Y^\theta_{t_i-, \xi_j^i}(\zeta_l) ) \right\rvert_2^2 \\
    &\le \left( \lvert X_{t_i,\xi_j^i}(\zeta_l)\rvert_2 + \lvert Y^\theta_{t_i,\xi_j^i}(\zeta_l) \rvert_2 \right)^2 + \left( \lvert X_{t_i,\xi_j^i}(\zeta_l)\rvert_2 + \lvert Y^\theta_{t_i-,\xi_j^i}(\zeta_l) \rvert_2 \right)^2 \\
    &\le 2\left( 2\lvert X_{t_i,\xi_j^i}(\zeta_l)\rvert_2^2 + \lvert Y^\theta_{t_i,\xi_j^i}(\zeta_l) \rvert_2^2 + \lvert Y^\theta_{t_i-,\xi_j^i}(\zeta_l) \rvert_2^2 \right) \\
	&\leq 4\lvert X_{t_i,\xi_j^i}(\zeta_l)\rvert_2^2 + 4B(m,T)^2, 
\end{align*}
which validates the condition of \Cref{lemma:convergencelocallyuniform},
\begin{equation}
\label{equ:dominating bound loss function}
\E_{\P}\left[\sup_{\theta \in \Theta_m} h(\theta, \zeta_l)\right] 
\leq   \E_{\P}\left[\frac{1}{n(\zeta_l)}\sum_{i=0}^{n(\zeta_l)-1} \frac{1}{J_i(\zeta_l)}\sum_{j=1}^{J_i(\zeta_l)} 4\lvert X_{t_i,\xi_j^i}(\zeta_l)\rvert_2^2 + 4B(m,T)^2 \right]  < \infty,
\end{equation}
by Assumption~\ref{assumption:5}. In the case that the readout network is bounded, the result is trivial.
This, together with the continuity of $\theta \mapsto h(\theta, \zeta_1)$ from Lemma~\ref{lem:properties for MC conv thm}, implies that Lemma~\ref{lemma:convergencelocallyuniform} can be applied to yield almost surely for $N \to \infty$ that the function 
\begin{equation}\label{equ:unif conv 1}
\theta \mapsto \frac{1}{N} \sum_{l=1}^{N} h(\theta, \zeta_l) = \hat\Phi_N(\theta)
\end{equation}
converges uniformly on $\Theta_m$ to 
\begin{equation}\label{equ:unif conv 2}
\theta \mapsto \E_{\P}[h(\theta, \zeta_1)] = \Phi(\theta).
\end{equation} 

Moreover, Lemma~\ref{lemma:convergencelocallyuniform} yields that $d(\theta^{\min}_{m,N},\Theta^{\min}_m)\to 0$ almost surely when $N\to \infty$. 
Then there exists a sequence $(\hat\theta^{\min}_{m,N})_{N \in \N}$ in $\Theta_m^{\min}$ such that $\lvert \theta^{\min}_{m,N} - \hat\theta^{\min}_{m,N} \rvert_2 \to 0$ almost surely for $N \to \infty$. Treating these as fixed parameterizations obtained after training, the uniform continuity of the random functions $\theta \mapsto Y_{t,\xi}^{\theta}$ on ${\Theta}_m$ implies that for any fixed  deterministic and bounded evaluation or test set $\zeta_0$ (in the same space as the $\zeta_l$),
$$\lvert Y_{t,\xi}^{\theta^{\min}_{m,N}}(\zeta_0) - Y_{t,\xi}^{\hat\theta^{\min}_{m,N}}(\zeta_0) \rvert_2 \to 0 \text{ a.s. for all  } t \in [0,T] \text{ and } \xi\in\Xi\text{ as } N \to \infty.$$ 
By continuity of $\phi$ this yields $\lvert h(\theta^{\min}_{m,N}, \zeta_0) - h(\hat\theta^{\min}_{m,N}, \zeta_0) \rvert \to 0$ a.s.\ as $N \to \infty$.
Let $\zeta_0$ now be a random variable which is independent of and identically distributed as the $\zeta_l$ defined on a copy $(\Omega_0, \mathbb{F}_0, \mathcal{F}_0, \P_0)$ of the filtered probability space $(\Omega, \mathbb{F}, \mathcal{F}, \P)$. Then the above statements hold for each realized evaluation set $\zeta_0(\omega_0)$ for $\P_0$-a.e.\ fixed $\omega_0$.
Now, considering the weights to be learned from the training set $\zeta_l(\omega)$, we have for $\P$-a.e.\ fixed $\omega \in \Omega$, that $\lvert h(\theta^{\min}_{m,N}\omb, \zeta_0) - h(\hat\theta^{\min}_{m,N}\omb, \zeta_0) \rvert \to 0$ $\P_0$-a.s.\ as $N \to \infty$.
With \eqref{equ:dominating bound loss function} we can apply dominated convergence which yields
\begin{equation*}
\lim_{N \to \infty} \E_{\P_0}\left[ \lvert h(\theta^{\min}_{m,N}\omb, \zeta_0) - h(\hat\theta^{\min}_{m,N}, \zeta_0) \rvert \right] = 0 \text{ for $\P$-a.e. } \omega \in \Omega.
\end{equation*}
Since for every integrable random variable $Z$ we have $0 \leq \lvert \E[Z] \rvert \leq \E[\lvert Z \rvert] $ and since $\hat\theta^{\min}_{m,N}\in \Theta_m^{\min}$ we can deduce that for $\P$-a.e.\ fixed $\omega \in \Omega$,
\begin{equation}
\label{equ: MC convergence}
\lim_{N \to \infty} \Phi(\theta^{\min}_{m,N}\omb) = \lim_{N \to \infty} \E_{\P_0}\left[  h(\theta^{\min}_{m,N}\omb, \zeta_0) \right] = \lim_{N \to \infty} \E_{\P_0}\left[  h(\hat\theta^{\min}_{m,N}\omb, \zeta_0) \right] = \Phi(\theta^{\min}_m).
\end{equation}
By the triangle inequality, for $\P$-a.e.\ fixed $\omega$, we have for $\hat\Phi_{\tilde N}$ and $\Phi$ evaluated on $\tilde N$ test samples $\zeta_0$ on $\Omega_0$, i.i.d. of the training samples $\zeta_l(\omega)$ yielding $\theta^{\min}_{m,N}\omb$ for $\omega\in\Omega$, that
\begin{equation}\label{equ: MC convergence 2}
\lvert \hat\Phi_{\tilde N}(\theta^{\min}_{m, N}\omb) -  \Phi(\theta^{\min}_{m}) \rvert \leq \lvert \hat\Phi_{\tilde N}(\theta^{\min}_{m, N}\omb) -  \Phi(\theta^{\min}_{m, N}\omb) \rvert + \lvert \Phi(\theta^{\min}_{m, N}\omb) -  \Phi(\theta^{\min}_{m}) \rvert.
\end{equation}
\eqref{equ:unif conv 1} and \eqref{equ:unif conv 2} imply that the first term on the right hand side converges to 0 when test samples $\tilde N \to \infty$ and \eqref{equ: MC convergence} implies that the second term on the right hand side converges to 0 a.s.\ when training samples $ N \to \infty$.  
Moreover, the uniform convergence in \eqref{equ:unif conv 1} and \eqref{equ:unif conv 2} yields the same result when setting $\tilde N = N$. Furthermore, Lemma~\ref{lemma:convergencelocallyuniform} yields the same result for $\hat\Phi_N(\theta^{\min}_{m, N})\omb$, i.e., when $\hat\Phi_{N}$ and $\Phi$ are defined through the $\zeta_l$ on the probability space corresponding to the training data. This finishes the proof of the first part of the Theorem.

To prove the convergence in $d_k$, we define $N_0 := 0$ and for every $m \in \N$
\begin{equation*}
N_m\omb := \min\left\{ N \in \N \; \vert \; N > N_{m-1}\omb, \lvert \Phi(\theta^{\min}_{m,N}\omb) - \Phi(\theta^{\min}_{m}) \rvert  \leq \tfrac{1}{m} \right\},
\end{equation*}
which is possible due to \eqref{equ: MC convergence} for $\P$-a.e.\ $\omega \in \Omega$. Then Theorem \ref{thm:1} implies that for $\P$-a.e.\ $\omega \in \Omega$
\begin{equation*}
\lvert \Phi(\theta^{\min}_{m,N_m\omb}\omb) - \Psi(\hat{X}) \rvert  \leq \tfrac{1}{m} +  \lvert \Phi(\theta^{\min}_{m}) - \Psi(\hat{X}) \rvert \xrightarrow{m \to \infty} 0.
\end{equation*}
Applying the same arguments as in deriving (\ref{eq: convergence in metric}) of Theorem~\ref{thm:1} yields the final result that
\begin{equation*}
d_k \left(  \hat{X} , Y^{\theta^{\min}_{m,N_m\omb}\omb}  \right)
\leq   \frac{\sqrt{2} c_0(k) \, c_1 }{c_2 c_3(k)} \, \left( \Phi(\theta^{\min}_{m,N_m\omb}\omb) - \Psi(\hat{X}) \right)^{1/2}  \xrightarrow{m \to \infty} 0,
\end{equation*}
for every $0 \leq k \leq K$ and for $\P$-a.e.\ $\omega \in \Omega$.
\end{proof}

\begin{rem}
    Note that since a neural network consisting of Lipschitz activations is bounded linearly in its input, the assumption of bounded neural networks can be replaced by assuming instead that the sum of squared inputs of the O-NJ-ODE is integrable. Since the network's input at a given time consists of all previous observations, this can be succinctly captured by assuming that the sum of the squares of all observations is integrable.  
\end{rem}

Explicitly relating the convergence results from \Cref{thm:MC convergence Yt} to the conditional expectation also yields the following corollary.


\begin{cor}\label{cor:1}
In the setting of Theorem \ref{thm:MC convergence Yt}, it also holds that $P$-a.s,
\begin{equation*}
\Phi(\theta^{\min}_{m,N_m}) \xrightarrow{m \to \infty} \Psi(\hat{X}) \quad \text{and} \quad \hat\Phi_{\tilde{N}_m}(\theta^{\min}_{m,\tilde{N}_m}) \xrightarrow{m \to \infty} \Psi(\hat{X}),
\end{equation*}
where $(\tilde{N}_m)_{m \in \N}$ is a suitable increasing random sequence in $\N$.
\end{cor}

\begin{proof}
The first convergence result is already shown in the final arguments proving Theorem \ref{thm:MC convergence Yt}. The second can be derived similarly, defining $\tilde{N}_0 := 0$ and $\tilde{N}_m$ for every $m \in \N$ by
\begin{equation*}
\tilde{N}_m := \min\left\{ N \in \N \; \vert \; N > \tilde{N}_{m-1}, \lvert \hat\Phi_N(\theta^{\min}_{m,N}) - \Phi(\theta^{\min}_{m}) \rvert  \leq \tfrac{1}{m}  \right\}.
\end{equation*}
After following the same arguments, this yields the result due to \eqref{equ: MC convergence 2}.
\end{proof}

\begin{rem}
    The assumption of \Cref{thm:1,thm:MC convergence Yt} to find elements in the sets of minimizers, $\Theta^{\min}_{m}$ and $\Theta^{\min}_{m,N}$, can be weakened. In particular, we can equivalently assume to have nearly-optimal elements $\theta^{\star}_m \in \{ \theta \in \Theta_m \, | \, \Phi(\theta) \leq \min_{\theta \in \Theta_m} \Phi(\theta) + \frac{1}{m} \}$ (and similarly for $\Theta^{\min}_{m,N}$) to retain the same results. 
\end{rem}



\section{Experiments}\label{sec:Experiments}
Having established convergence guarantees for the O-NJ-ODE, we now validate its performance on synthetic datasets. The exploration focuses on simple processes that feature one-dimensional domains and outputs. Moreover, these processes are continuous, Markovian, and martingales for fixed points in the function domain. Although the model is, in theory, capable of predicting complex dynamics such as path-dependence and discontinuities, the focus on simple dynamics isolates the model's core contribution, namely to generalize predictions across the spatial domain. This controlled setting enables direct analysis of the operator framework's ability to interpolate and extrapolate in space. A wide range of datasets exhibiting different forms of space-time coupling and spatial structures were generated for testing; refer to \url{https://github.com/OliLot/ONJODE} for the codebase. Here, we focus on the Brownian cosine process $X(t,\xi)=W_t\cos(\xi)$ and simulate paths on a discretized grid. It serves as a representative dataset for this investigation, featuring multiplicative coupling and a consistent spatial structure. Dataset configurations are shown in \Cref{tab:dataset_specs}.

\begin{table}[h]
\centering
\caption{Synthetic dataset specifications}
\label{tab:dataset_specs}
\begin{tabular}{@{}lc@{}}
\toprule
\textbf{Specification} & \textbf{BrownianCosine}  \\
\midrule
Process & $X(t,\xi) = W_t \cos(\xi)$ \\
\midrule
Number of paths ($N$) & 20,000  \\
Time points ($N_t)$ & 51 \\
Terminal time ($T$) & 1.0  \\
Observation probability in time ($p_t$) & 0.1 \\
Space points ($N_s$) & 51 \\
Space domain ($\Xi$) & $[-5, 5]$ \\
Observation probability in space ($p_s$) & 0.3 \\
\bottomrule
\end{tabular}
\end{table}

We use a temporal and spatial grid,  sampling for each time point from a $\text{Bernoulli}(p_t)$ distribution to determine whether it is an observation time. At each of these samples, spatial observation points are similarly drawn from the grid by sampling from the distribution $\text{Bernoulli}(p_s)$. Together, they define the observed coordinates. \Cref{fig:BrownianCosData} shows an example path of the simulated test set in its entirety, along with a scatter plot indicating the observed data points. The O-NJ-ODE model successively gets access to observations as it is evaluated through time.

\begin{figure}[h]

\begin{subfigure}{0.5\textwidth}
\includegraphics[width=1\linewidth, height=6cm]{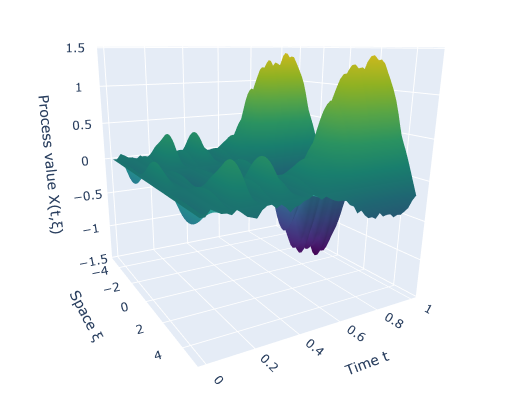} 
\caption{$X(t, \xi) = W_t\cos(\xi)$ path}
\label{fig:subim1}
\end{subfigure}
\hfill
\begin{subfigure}{0.5\textwidth}
\includegraphics[width=0.99\linewidth, height=5.5cm]{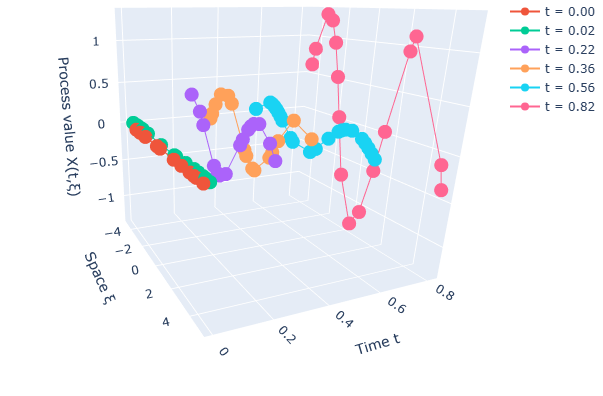}
\caption{Observed coordinates}
\label{fig:subim2}
\end{subfigure}

\caption{(a) Example path of the Brownian cosine process, from $t=0$ to terminal time $t=T$. (b) Discrete and irregularly sampled observations for this path. The time step is $dt=0.02$ and the space step $ds=0.2$}
\label{fig:BrownianCosData}
\end{figure}

Following the theoretical construction, our neural network architecture utilizes the locally aggregating generalized kernel (\ref{eq:gen kernel}). Additionally, the readout network consists of tanh activations for compatibility with \Cref{thm:MC convergence Yt}, while all other modules contain ReLU activations.  Each of the five modules is parameterized as two hidden layer neural networks (with linear output layers) for the sake of reducing complexity and removing architectural confounding factors. All modules except the generalized kernel networks have 50 nodes in each hidden layer, whereas these networks have 32 nodes owing to their smaller input dimensions. To stabilize training, inputs for the Neural ODE network are normalized with tanh scaling, which prevents exploding and vanishing gradients. Layer Normalization is also introduced in all but the readout network, which has strong evidence of stabilizing learning for sequential models \citep{ba2016layernorm}, as it conditions the optimization landscape and encourages gradient flow. The complete architecture along with training specifications are given in \Cref{app:implementation}; this defines our baseline model. We additionally conduct a sensitivity analysis examining the impact of activation function choice, sample size, and observation sparsity on model performance.

Model performance is assessed with a similar evaluation metric to \citep{herrera2021neural}, adjusted to account for the additional space dimension. It aggregates the mean squared errors between model predictions and the optimal true conditional expectations, which can be computed since the synthetic data generating process is known (however, not to the model, which only has access to observations). More precisely, the metric is evaluated on the test set and is given by
\begin{equation*}
    \text{MSE}(Y^{\theta},\hat{X}) := \frac{1}{N_{test}}\sum_{l=1}^{N_{test}}\frac{1}{N_t}\sum_{i=0}^{N_t-1}\frac{1}{N_s}\sum_{j=0}^{N_s-1}\left((Y^{\theta}_{k_i,x_j})^{(l)}-\hat{X}_{k_i,x_j}^{(l)}\right)^2,
\end{equation*}
where $k_i = \frac{iT}{N_t-1}$ and $x_j=\frac{j(R-L)}{N_s-1}+L$ for space domain $\Xi = \left[L, R\right]\subset \R$. Additionally, we compare the model's test loss with the optimal test loss, which is calculated by substituting the known conditional expectation into the empirical loss function (\ref{equ:appr loss function}) and evaluating it over the test set. In the limit, this is the minimal loss, up to indistinguishability with respect to the pseudo-metric $d_k$ (\ref{equ: pseudo metric}).

\subsection{Results and Discussion}
We find that the O-NJ-ODE has a strong predictive performance across both the space and time domains, as demonstrated by the example path in \Cref{fig:BrownianCosPath}. The model captures the spatio-temporal structure of the cosine-scaled Brownian motion. At observation times, we observe an update in predictions matching those of the true conditional expectation, indicating that the jump network adequately captures the conditional expectation as new information becomes available. Jump predictions are also propagated through time until the next observation, which matches the martingale characteristic of the process. This shows that the evolution network accurately captures the dynamics of the true conditional expectation between observations. Model predictions at both observed and unobserved times have a low magnitude of error from the true conditional expectation. However, since the loss function (\ref{equ:appr loss function}) does not capture the scale of these errors relative to the true conditional expectation values, training strictly prioritizes correcting larger errors over smaller ones. As a result, prediction errors are homogeneous over time regardless of the magnitudes of the conditional expectation. A scale adjusted loss function could be used if the relative accuracy of predictions is important. However, the analysis and empirical evaluation of this is left for future research.

\begin{figure}[tb]
\centering
\includegraphics[width=0.85\linewidth, height=21cm]{figures/BrownianCosPath3_1.pdf} 
\caption*{(a) Path over time starting from initialization at zero ($W_0=0$), including the first of the additional observation times in the third panel, at $t=0.04$.}
\label{fig:BrownianCosPath3,1}
\end{figure}

\begin{figure}[tb]
\centering
\includegraphics[width=0.85\linewidth, height=21cm]{figures/BrownianCosPath3_2.pdf} 
\caption*{(b) Continued until final observation time.}
\caption{Evolution of the $X(t,\xi)=W_t\cos(\xi)$ path of minimal MSE, showcasing the O-NJ-ODE prediction and target conditional expectation. Although not shown here, performance is maintained until terminal time $T$.}
\label{fig:BrownianCosPath}
\end{figure}

Training results indicate that the O-NJ-ODE converges efficiently and with good stability. As shown in \Cref{fig:BrownianCosEval(0.3)}, both train and test losses converge rapidly, with most learning taking place within the first 15 epochs, after which they oscillate around the optimal loss of $6.675\times10^{-2}$ (calculated using the true conditional expectation). Furthermore, the train and test losses track each other closely throughout, indicating no significant overfitting. After twenty epochs, the losses plateau and improvements are negligible. Ultimately, the test loss achieves a minimal value of $6.983\times10^{-2}$ at epoch 86, and the evaluation metric reaches its minimum of $8.45\times10^{-5}$ at epoch 81. The initial spike in the train loss likely reflects the model's adjustment from random initialization. Subsequent oscillations are small and remain around the optimal loss. Decreasing the learning rate or increasing batch-size could promote further stability. However, the evaluation metric indicates successful convergence and suggests that the oscillations do not impact model performance.

\begin{figure}[h]
    \centering
    \includegraphics[width=1\linewidth]{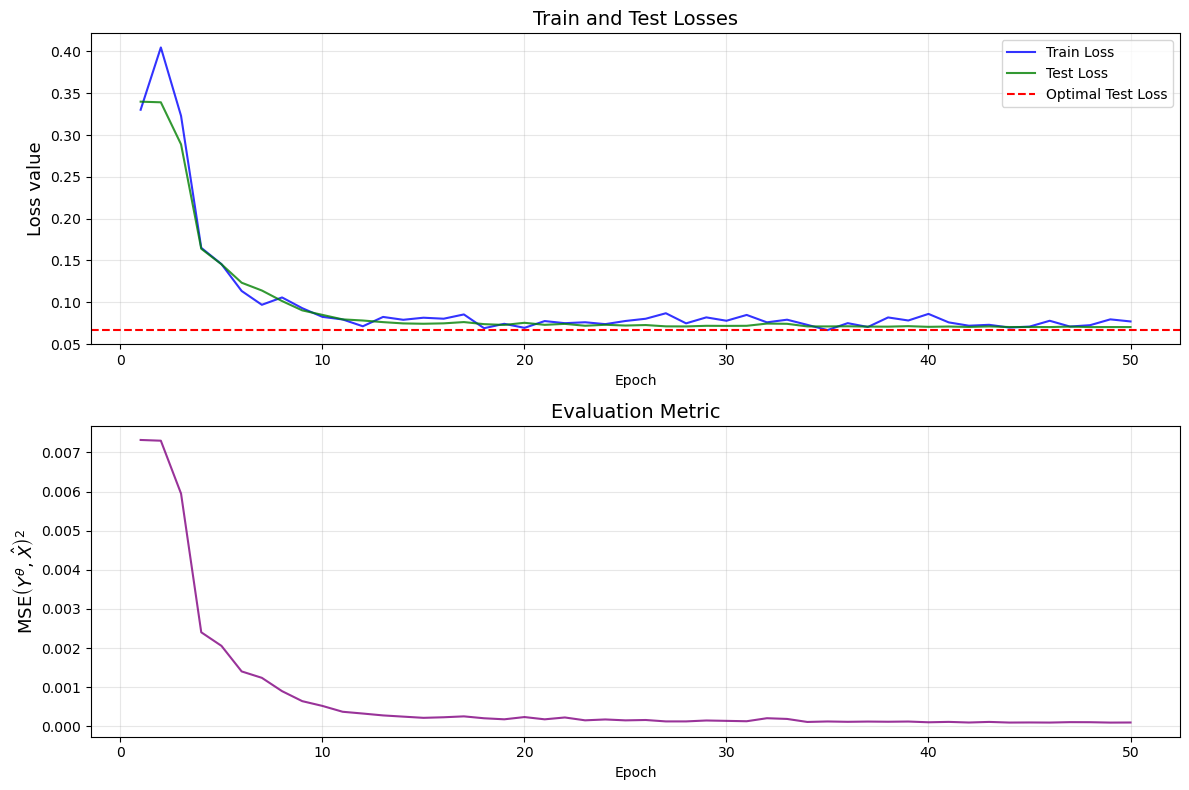}
    \caption{O-NJ-ODE training diagnostics for the BrownianCosine dataset.}
    \label{fig:BrownianCosEval(0.3)}
\end{figure}

Empirically, dismissing the use of bounded activation functions in the readout network or using alternatives such as LeakyReLU, GeLU, or only tanh in the other networks has a modest impact on absolute prediction accuracy. However, the inductive biases they introduce, such as eliciting smoothness (GeLU, tanh), piecewise linearity (ReLU, Leaky ReLU), or boundedness (tanh, sigmoid), affect characteristics of the model predictions. When replacing the tanh activations of the readout network with LeakyReLU, we observe a greater variance of the model predictions. \Cref{app:leakyReLU} shows the path of minimal evaluation error for this model, as well as the corresponding loss and MSE plots. The model achieves a minimal evaluation metric of $1.91\times10^{-4}$ with a minimal test loss of $7.310\times10^{-2}$. Compared to the original model, the MSE has doubled and the loss increased by about $4\%$. Although there is a significant relative degradation to the baseline model, this is low in absolute terms. Moreover, the model is trained at a slower rate and with decreased stability. The induced prediction variance adversely impacts the model's ability to estimate the smooth cosine. In practice, it is therefore beneficial to tailor the choice of activation functions to the predicted process. This helps the model to generalize effectively and obtain a more representative output.

Reducing the number of sample paths degrades training efficiency and performance relative to the baseline, yet maintains a reasonable quality in absolute terms. Using $N=5000$ paths yields a minimal MSE of $2.14\times10^{-4}$ at epoch 96 and a minimal test loss of $7.295\times10^{-2}$ at epoch 95. Compared to the baseline model, this is an approximately $2.5$-fold increase in MSE. Moreover, the test loss deviation from the optimal value doubles to around $9\%$. The corresponding loss and prediction plots can be seen in \Cref{app:reduced sample}. 

Relative performance degradation is more pronounced with sparse spatial observations. With a reduced spatial observation probability of $p_s=0.1$, the model obtains a minimal MSE of $9.15\times10^{-4}$ at epoch $92$ and minimal test loss of $8.102\times10^{-2}$ at epoch $100$ (\Cref{app:sparse samples}). While this performance is acceptable, it represents the greatest relative decrease from baseline; the test loss is more than $20\%$ from the optimum, and the MSE increases tenfold. These results demonstrate the sensitivity to data coverage and the amount of training samples, as implied by our theoretical results. Nevertheless, despite notable relative degradation, the model maintains reasonable performance under limited data availability. Rapidly converging losses also indicate stable learning, albeit with larger oscillations. \Cref{tab: eval comparison} summarizes results from the sensitivity analysis.



\begin{table}[h]
\centering
\begin{tabular}{|c|c|cccc|}
\hline
     & Optimal & Baseline & Leaky ReLU & Small Sample Size & Sparse Observations \\ \hline
    MSE & 0 &$8.45\times10^{-5}$ & $1.91\times10^{-4}$ & $2.14\times10^{-4}$ & $9.15\times10^{-4}$\\
    Test Loss & $6.675\times10^{-2}$ & $6.983\times10^{-2}$ & $7.310\times10^{-2}$ & $7.295\times10^{-2}$ & $8.102\times10^{-2}$\\
\hline
\end{tabular}
\caption{MSE and test losses of the baseline model (as specified in \Cref{tab:dataset_specs}) and its modifications: Leaky ReLU readout activations, reduced sample size ($N=5000$), and sparse spatial observations ($p_s=0.1$).}
\label{tab: eval comparison}
\end{table}

\section{Conclusion}

We elevate the Neural Jump ODE framework to the more general setting of optimally predicting function-valued stochastic processes taking values in $L^2(\Xi, \R^{d_X})$. This extension introduces a generalized kernel that processes information and captures spatial context. An RNN structure is leveraged to capture path dependence in place of the signature. Under the weaker assumption of measurability (as opposed to continuity) of the target functions, the O-NJ-ODE encapsulates previous NJ-ODE models as a special case and extends their theoretical results even in the finite-dimensional case. We introduce a mathematically rigorous problem formulation for the spatio-temporal setting and prove novel convergence guarantees to the $L^2$-optimal predictor. Synthetic experiments on a cosine-scaled Brownian motion validate these theoretical results, demonstrating strong generalization in both space and time, robustness to the choice of activation functions, and sustained performance with sparse observations or limited sample sizes. 

Our controlled experiments establish foundational feasibility, but further exploration of complex dynamics, such as path-dependent or discontinuous processes, regime shifts, and higher-dimensional settings, would further strengthen the empirical validation of the O-NJ-ODE capacity. It is worth noting that previous extensions of the NJ-ODE, such as dependent observation frameworks or input-output systems, are transferable to this setting. The O-NJ-ODE is broadly applicable to real world stochastic processes, ranging from yield curves and implied volatility surfaces to health monitoring such as EEG measurements, where irregularly sampled and discrete observations arise naturally in function-valued processes. Exploring such real settings would validate practical effectiveness and highlight performance under the complicated dynamics, regime shifting, and data scarcity common in these domains.



\if\addackn1
	\section*{Acknowledgement}
 
\fi

\bibliographystyle{iclr2021_conference}
\bibliography{references.bib}
\clearpage

\appendix
\section*{Appendix}

\section{Experimental Details}\label{sec:Experimental Details}
	
\subsection{Training and Baseline Model Specifications}\label{app:implementation}

\begin{table}[h]
\centering
\caption{Training specifications for O-NJ-ODE models}
\label{tab:training_specs}
\begin{tabular}{@{}lc@{}}
\toprule
\textbf{Parameter} & \textbf{Value} \\
\midrule
\multicolumn{2}{@{}l}{\textit{Training Configuration}} \\
Epochs & 100 \\
Batch size & 200 \\
Train/test split & 80/20 \\
Learning rate & 0.001 \\
Optimizer & Adam \\
Scheduler & Reduce on Plateau \\ 
ODE solver & Euler \\
\midrule
\multicolumn{2}{@{}l}{\textit{Network Architecture}} \\
NODE network ($f_{\theta_1}$) & $(50, \text{ReLU}) \times 2$ \\
Encoder network ($\rho_{\theta_2}$) & $(50, \text{ReLU}) \times 2$ \\
Generalized kernel networks ($\phi^1_{\theta_3}, \phi^2_{\theta_3}$) & $(32, \text{ReLU}) \times 2$ \\
Readout network ($g_{\theta_4}$) & $(50, \tanh) \times 2$ \\
Hidden state size & 100 \\
Kernel context size & 100 \\
\midrule
\multicolumn{2}{@{}l}{\textit{Regularization \& Preprocessing}} \\
NODE input scaling function & $\tanh$ \\
layer normalization & Enabled \\
\bottomrule
\end{tabular}
\end{table}

\clearpage

\section{Additional Model Implementations}\label{app:BrownianCosPredictions}
\subsection{LeakyReLU Readout Activations}\label{app:leakyReLU}

\begin{figure}[h]
    \centering
    \includegraphics[width=1\linewidth]{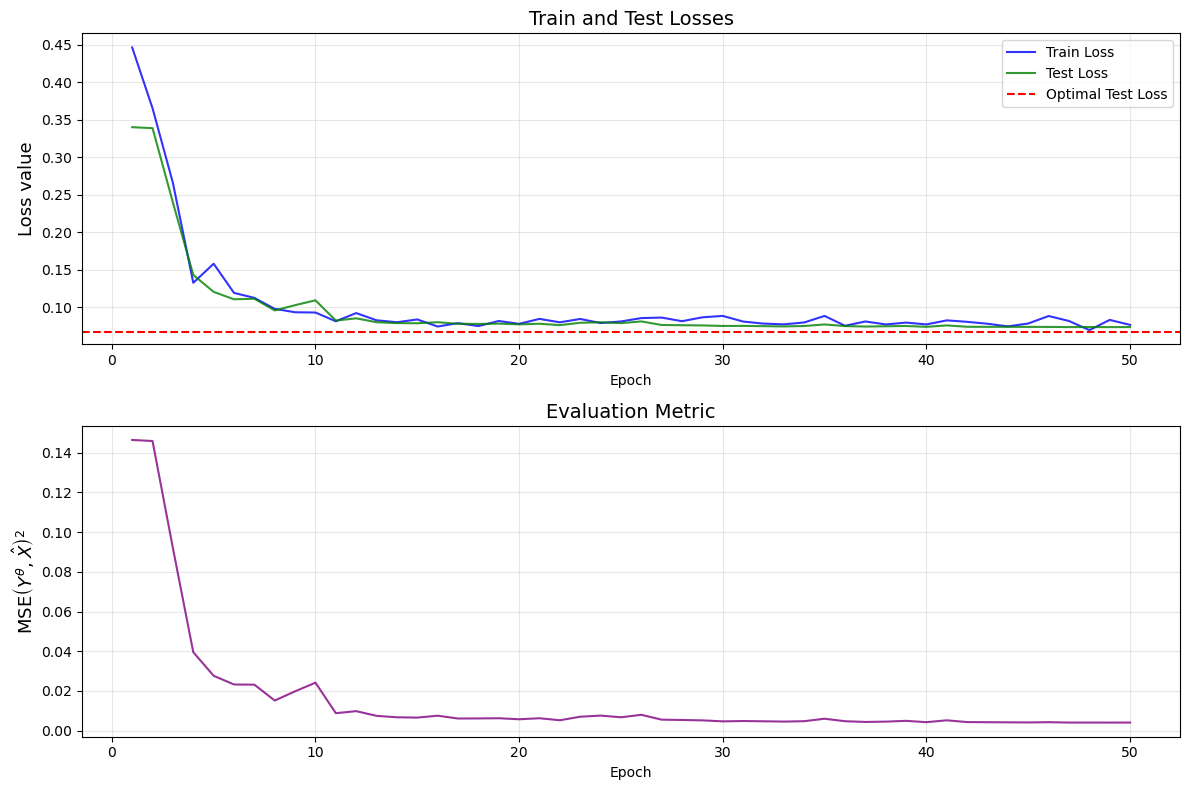}
    \caption{Leaky ReLU activations in the readout result in marginally slower and less stable training, with delayed plateauing of both the validation loss and evaluation metric at larger values.}
\end{figure}

\begin{figure}[h]
\centering
\includegraphics[width=0.9\linewidth, height=22cm]{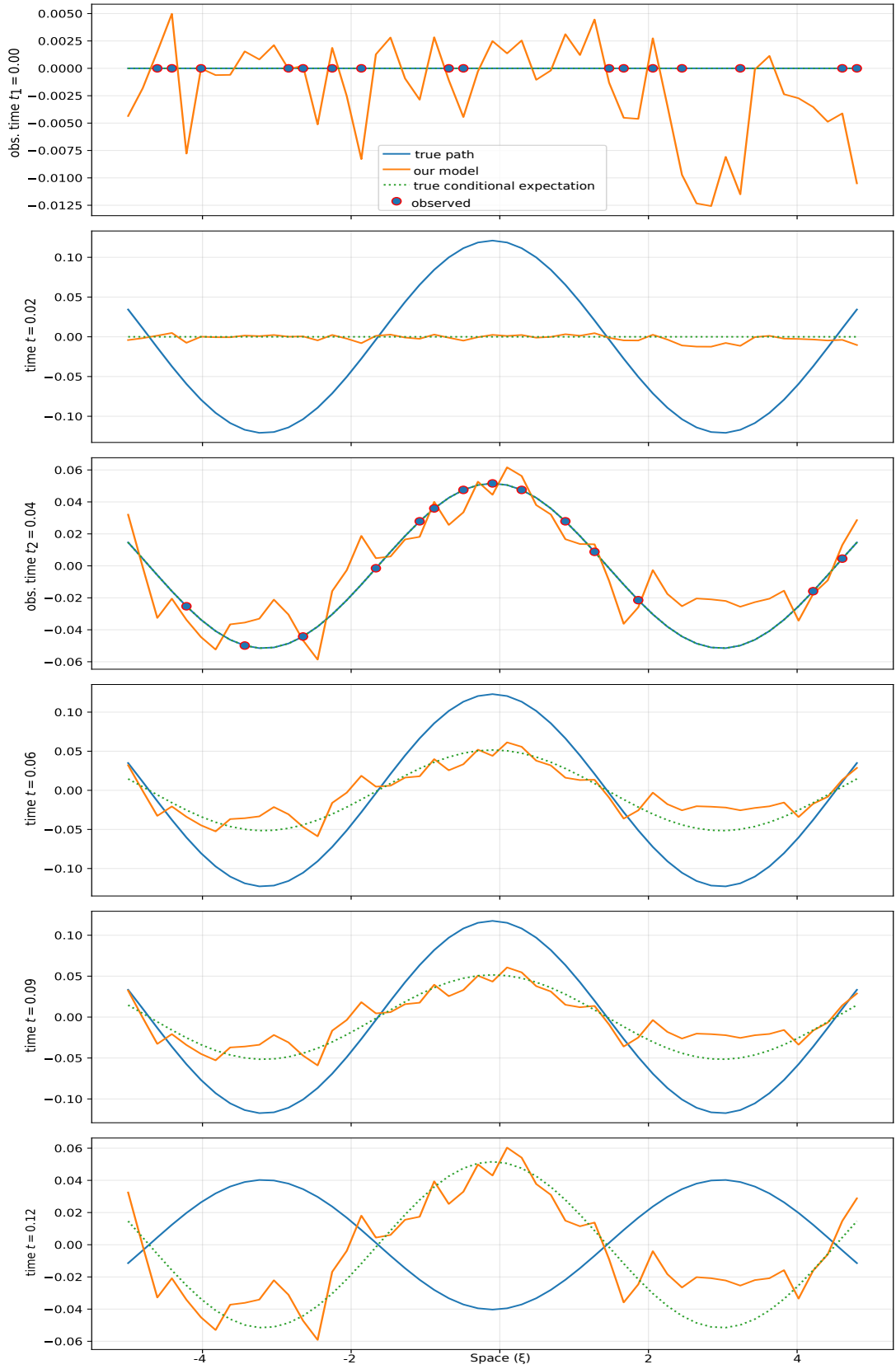} 
\caption{Replacing tanh activation functions in the readout results in increased roughness and prediction variance, but the model maintains its strong predictive ability.}
\label{fig:BrownianCosPath3,1 LeakyRelu}
\end{figure}

\clearpage

\subsection{Reduced Path Sample Size - $N=5000$}\label{app:reduced sample}

\begin{figure}[h]
    \centering
    \includegraphics[width=1\linewidth]{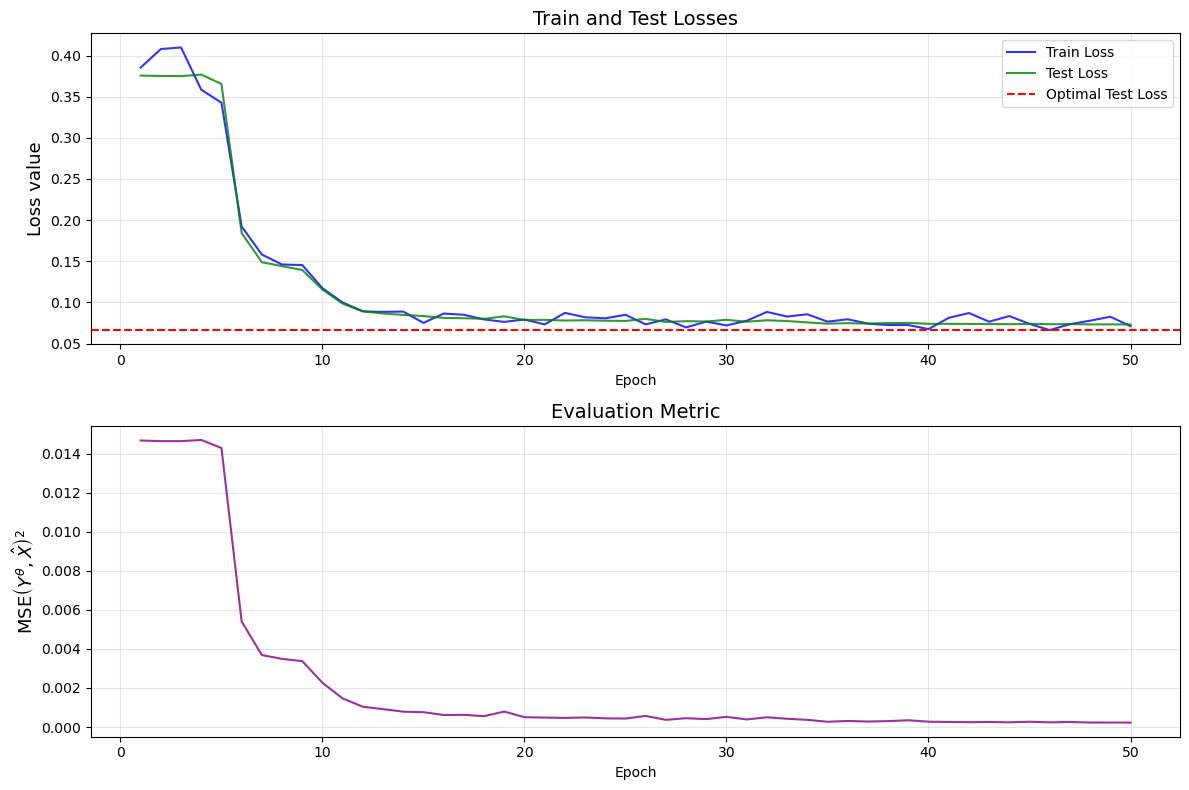}
    \caption{Reducing the sample size only marginally degrades training and performance in absolute terms.}
\end{figure}

\begin{figure}[h]
\centering
\includegraphics[width=0.9\linewidth, height=22cm]{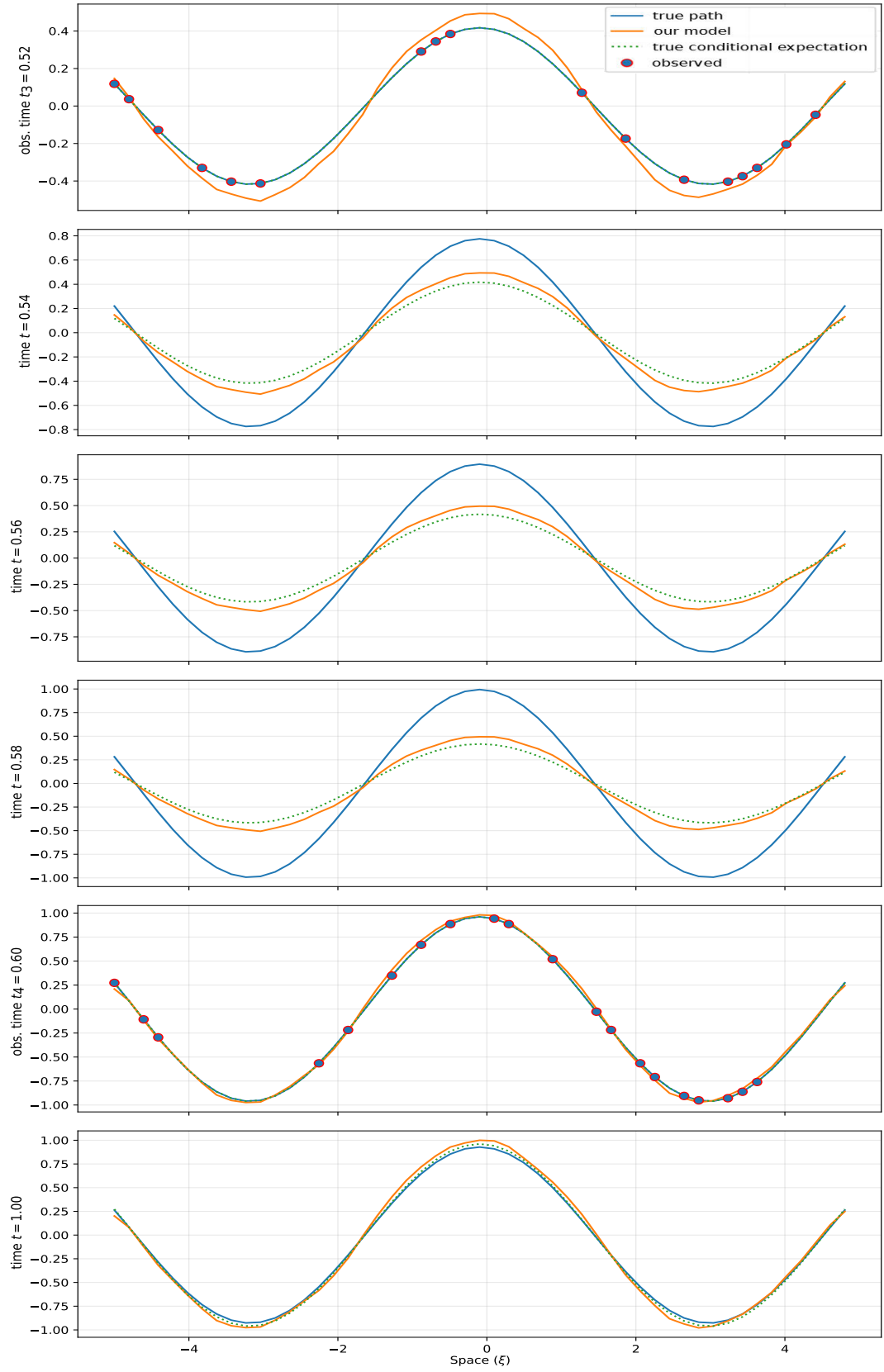} 
\caption{Despite the reduction in sample size to a quarter of the baseline, the decrease in prediction accuracy is negligible.}
\label{fig:BrownianCosPath3,1 reduced train size}
\end{figure}

\clearpage

\subsection{Sparse Spatial Observations - $p_s=0.1$}\label{app:sparse samples}

\begin{figure}[h]
    \centering
    \includegraphics[width=1\linewidth]{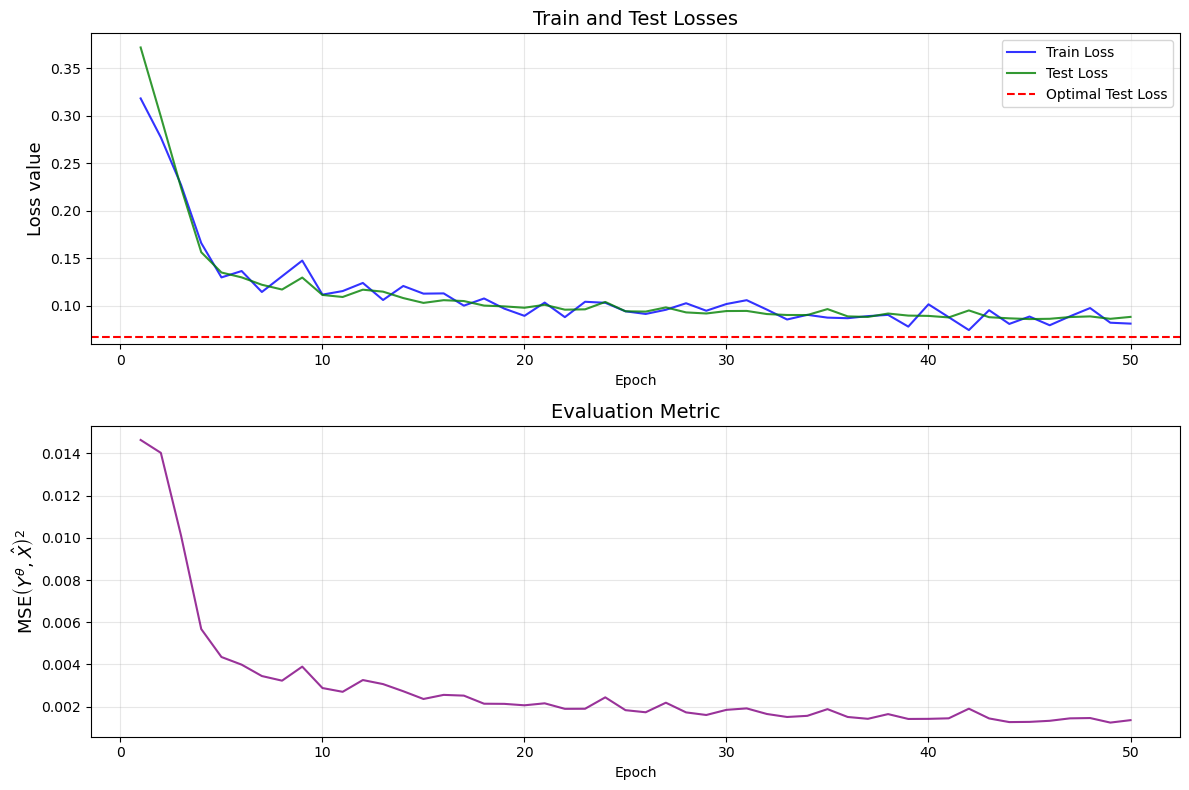}
    \caption{Although small in absolute terms, increasing spatial data sparsity has the greatest impact on both training speed and stability, as well as its predictive performance.}
\end{figure}

\begin{figure}[h]
\centering
\includegraphics[width=0.9\linewidth, height=22cm]{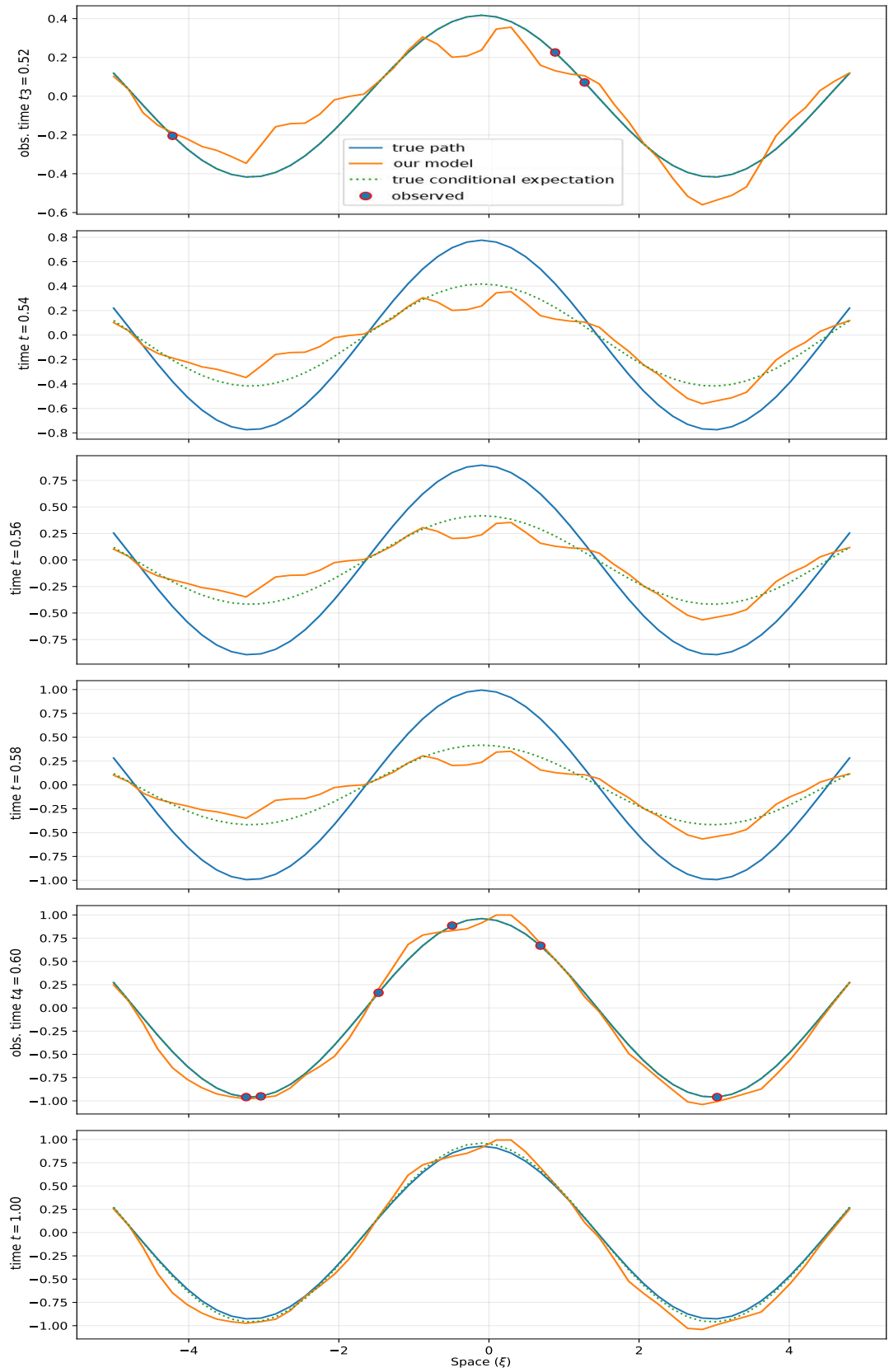} 
\caption{Prediction performance is noticeably weaker with the decreased access to spatial information, but the model still generalizes well.}
\label{fig:BrownianCosPath3,1 scarse spatial observations}
\end{figure}

\if\inclapp1



\fi


\end{document}